\documentclass{article}
\usepackage{arxiv}
\usepackage{graphicx}
\usepackage{xspace}
\usepackage{epsfig}
\usepackage{color}
\usepackage{url}
\usepackage{tikz}
\usepackage{float}
\usepackage{booktabs}
\usepackage{makecell}
\usepackage{multicol}
\usepackage{longtable}
\usepackage{array}
\usepackage{anyfontsize}
\usepackage{balance}
\usepackage{algpseudocode}
\usepackage{amsthm}
\usepackage{natbib}
\renewcommand{\headeright}{}
\usepackage[vlined,linesnumbered,ruled,titlenotnumbered]{algorithm2e}
\usepackage{mathtools}
\usepackage{hyperref}
\usepackage{comment}
\usepackage{algpseudocode,booktabs,amsmath}
\usepackage{multirow, rotating}
\usepackage{mwe}
\usepackage{blindtext}
\usepackage{amssymb}
\usepackage{hyperref}
\usetikzlibrary{shapes,decorations,patterns,positioning}
\usetikzlibrary{plotmarks}
\usepackage{booktabs}
\usepackage{pdflscape}
\usepackage{adjustbox}
\usepackage{capt-of}
\usepackage{changepage}
\usepackage{amsfonts}
\usepackage{fancyhdr}
\usepackage{subfigure}
\usepackage{dsfont}
\usepackage{color}
\usepackage{xcolor}
\usepackage{colortbl}

\newtheorem{lemma}{Lemma}
\newtheorem{theorem}{Theorem}

\newcommand{\ignore}[1]{}

\newcommand{\kopt}{k_{opt}}

\DeclareMathOperator*{\argmax}{arg\,max}

\newcommand{\R}{\mathbb{R}}

\newcommand{\ga}{GA\xspace}
\newcommand{\gga}{GGA\xspace}
\newcommand{\gsemo}{GSEMO\xspace}
\newcommand{\nsga}{NSGA-II\xspace}
\newcommand{\spea}{SPEA2\xspace}

\AtBeginDocument{
  \providecommand\BibTeX{{
    \normalfont B\kern-0.5em{\scshape i\kern-0.25em b}\kern-0.8em\TeX}}}
\usepackage{xcolor}
\usepackage{tcolorbox}

\begin{document}

\title{Optimizing Monotone Chance-Constrained
Submodular Functions Using Evolutionary
Multi-Objective Algorithms}

\author{Aneta Neumann\\
Optimisation and Logistics\\
School of Computer and Mathematical Sciences\\
The University of Adelaide\\
Adelaide, Australia\\
 \And
Frank Neumann\\
Optimisation and Logistics\\
School of Computer and Mathematical Sciences\\
The University of Adelaide\\
Adelaide, Australia\\
}
\maketitle

\begin{abstract}
Many real-world optimization problems can be stated in terms of submodular functions. Furthermore, these real-world problems often involve uncertainties which may lead to the violation of given constraints. 
A lot of evolutionary multi-objective algorithms following the Pareto optimization approach have recently been analyzed and applied to submodular problems with different types of constraints. We present a first runtime analysis of evolutionary multi-objective algorithms based on Pareto optimization for chance-constrained submodular functions. Here the constraint involves stochastic components and the constraint can only be violated with a small probability of $\alpha$. 
We investigate the classical GSEMO algorithm for two different bi-objective formulations using tail bounds to determine the feasibility of solutions. 
We show that the algorithm GSEMO obtains the same worst case performance guarantees for monotone submodular functions as recently analyzed greedy algorithms for the case of uniform IID weights and uniformly distributed weights with the same dispersion when using the appropriate bi-objective formulation. As part of our investigations, we also point out situations where the use of tail bounds in the first bi-objective formulation can prevent GSEMO from obtaining good solutions in the case of uniformly distributed weights with the same dispersion if the objective function is submodular but non-monotone due to a single element impacting monotonicity.
Furthermore, we investigate the behavior of the evolutionary multi-objective algorithms GSEMO, NSGA-II and SPEA2 on different submodular chance-constrained network problems. Our experimental results show that the use of evolutionary multi-objective algorithms leads to significant performance improvements compared to state-of-the-art greedy algorithms for submodular optimization.
\end{abstract}

\section{Introduction}

Artificial-intelligence-based optimization techniques have found various applications in solving complex optimization problems. They provide some of the best approaches for prominent combinatorial optimization problems such as satisfiability~\citep{DBLP:journals/jair/XuHHL08,biere2013lingeling,rintanen2012planning}, the maximum coverage problem~\cite{DBLP:conf/gecco/NeumannA022} and the traveling salesperson problem~\citep{lin1973effective}.
Evolutionary algorithms are artificial-intelligence-based optimization algorithms that mimic the principle of natural evolution to solve optimization and machine learning problems. 
Evolutionary algorithms have been widely applied to solve complex optimization problems. 
In particular, this holds for complex engineering problems where the quality of solutions can often only be determined by simulations~\citep{liu2018survey,sebastian2012flywheel}. While evolutionary algorithms are well suited for solving broad classes of problems and often achieve good results within a reasonable amount of time, understanding their optimization behavior and determining the type of problems for which they can provably achieve high-quality solutions is still a challenging task.
The theory of evolutionary computation aims to explain such good behaviors and also point out the limitations of evolutionary computing techniques. A wide range of tools and techniques have been developed in the last 25 years and we point the reader to~\cite{BookDoeNeu,DBLP:books/daglib/0025643,Auger11,ncs/Jansen13} for comprehensive presentations.

Stochastic components play a crucial role in many real-world applications and uncertain components present a significant risk in terms of safety and/or profitability when dealing with important constraints~\citep{ben2009robust}. Chance constraints allow to model critical uncertainties and requirements based on uncertain situations~\citep{charnes1959chance}. A chance constraint involves random components and in this context it is required that the constraint is violated with a small probability of at most $\alpha$.
In real-world scenarios, it is often more beneficial to guarantee that a function value does not drop below a certain value with a high probability ~\citep{devore2011probability,miller2012probability,DBLP:journals/corr/abs-2102-05235}. For example, in the area of ore mine planning, the profits achieved within different years should not drop below certain values, as this would make the entire operation not viable and lead to bankruptcy for the company~\citep{fioroni2008concurrent,koushavand2014linear}. 
In the area of milling for mining, it is crucial to guarantee that the throughput rate of semi-autogenous grinding (SAG) mills does not drop below a certain value with a good probability which can cause an expensive mine shutdown and significant delays in mining operations~\citep{galan2002robust,zhou2016survey,itavuo2019mass}. In the electrical power industry, which is extremely vulnerable to disruption, it is important to guarantee a balance between supply and demand of fossil and renewable energy sources such as solar, hydro and wind at each time interval to avoid energy shortages~\citep{santoso2012electrical,strielkowski2021renewable,leonida2022mining}.

\subsection{Related Work}

In many real-world environments it is crucial to deliver solutions that are robust.~\citet{beyer2007robust} provided a discussion on different methods of performing robust optimization in practice, and addressed performance aspects and test scenarios for direct robust optimization techniques.

In certain real-world scenarios,  decision variables may encounter disturbances  prior to their deployment~\citep{gabrel2014recent,beyer2007robust}. In such cases, it is crucial that the solution not only exhibits optimality but also demonstrates a high degree of robustness against perturbations~\citep{marczyk2000stochastic}. This robustness ensures the solution's effectiveness and reliability in practical applications. \citet{DBLP:conf/ppsn/SinghB22} proposed an approach for black-box robust optimization and formulated a multi-objective optimization problem utilizing  discretized quantile functions. This formulation enables the identification of stochastically non-dominated solutions for decision-makers. To enhance efficiency and reduce the number of required function evaluations, the authors employed surrogate approximation techniques and leveraged historical data re-use.

\citet{DBLP:journals/tec/AsafuddoulaSR15} proposed several problem formulations for robust optimization, which aim to find trade-off solutions that balance feasibility and performance robustness across various levels for many-objective optimization algorithms. This approach for robust optimization was designed to provide a comprehensive set of trade-off solutions that cover different levels of feasibility and performance robustness. This was achieved by simultaneously utilizing systematically generated reference directions and applying epsilon level comparison to efficiently solve the problem.

Other aspects of robust optimization techniques focus on  assessing robustness within the context of multi-objective optimization. \citet{ferreira2008evolutionary} introduced an optimization technique that considers robustness in multi-objective optimization using a dispersion parameter. This parameter enables decision-makers to control the extent of the robust Pareto front. The approach involves identifying the robust region of the optimal Pareto front using multi-objective evolutionary algorithms.

\citet{DBLP:conf/ppsn/NeumannXN22} investigated the knapsack problem with uncertain profits and explored various methods for handling stochastic profits using tail inequalities. \citet{Meakpsp} explored the effectiveness of simple evolutionary algorithms and a problem-specific crossover operator in optimizing uncertain profits.

\subsubsection{Chance-Constrained Optimization}

Chance-constrained optimization has been widely studied in the operations research literature~\citep{charnes1959chance,miller1965chance,DBLP:journals/ior/ChenSST10,cooper2004chance,DBLP:journals/ior/LejeuneM16}. Such studies often require that the underlying optimization problem is linear and that the elements are chosen independently according to a normal distribution. This then allows to translate a chance-constrained problem into a deterministic counterpart, which can be  efficiently solved using solvers such as CPLEX~\citep{cplex2009v12} or Gurobi~\citep{gurobi}.
A crucial problem when considering chance-constrained problems where elements are not chosen according to a normal distribution is to decide whether a given solution is feasible. Even for very simple probability distributions, this problem is already NP-hard~\citep{ahmed2008solving,prekopa2011uniform}.
For chance constraints that are in general hard to evaluate exactly, but well-known tail inequalities such as Chernoff bounds and Chebyshev's inequality~\citep{MotwaniRaghavan} may be used to estimate the probability of a constraint violation.

Chance-constrained problems have gained significant attention in the field of evolutionary computation in recent years. Early studies focus on utilizing simulations and various sampling methods to address chance constraints optimization~\citep{DBLP:journals/eor/PoojariV08,He_Shao_ICIECS_2009, DBLP:journals/tec/LiuZFG13, Masutomi_2013_jaciii,Loughlin_Ranjithan_gecco99}. 

Recent work explored the use of tail bounds for packing problems with stochastic knapsack or quality constraints~\citep{DBLP:conf/ppsn/NeumannXN22, DBLP:conf/gecco/XieHAN019, DBLP:conf/gecco/XieN020,DBLP:conf/ecai/AssimiHXN020,MOEAwSWS,E23OMOEA,U3OEADCHKP,TCHTTP}.

\subsubsection{Submodular Optimization Problems}

In terms of the theoretical understanding and the applicability of evolutionary algorithms, it is desirable to be able to analyze them on a broad class of problems. It is important to design appropriate evolutionary techniques for such broad classes of problems~\cite{lovasz1983submodular,fujishige2005submodular}.
Submodular functions model a wide range of problems where the benefit of adding solution components diminishes as more elements are  included. This property makes them suitable for capturing diminishing returns and interactions among elements  present in various real-world applications. Extensive studies have been conducted in the literature to explore the properties~\citep{Nemhauser:1978,DBLP:journals/mp/NemhauserWF78,vondrak2010submodularity,DBLP:books/cu/p/0001G14,DBLP:conf/icml/BianB0T17,DBLP:conf/kdd/LeskovecKGFVG07,pmlr-v65-feldman17b,pmlr-v97-harshaw19a} and model a variety of real-word applications~\citep{DBLP:conf/aaai/KrauseG07,DBLP:conf/stoc/LeeMNS09,golovin2011adaptive,DBLP:conf/aaai/MirzasoleimanJ018}.

In recent years, the design and analysis of evolutionary algorithms for submodular optimization problems has gained increasing attention. It has been shown in particular that Pareto optimization approaches which use a bi-objective formulation are highly successful when dealing with a wide range of submodular problems. A bi-objective formulation in evolutionary algorithms for submodular optimization problems incorporates the  goal function as one objective and the given constraint as the other objective. This approach enables then a greedy behavior of evolutionary multi-objective algorithms which leads to provably good solutions for a wide range of problems. Variants of the algorithm \gsemo~\citep{Giel2003} have been widely studied in the area of runtime analysis in the field of evolutionary computation~\citep{DBLP:journals/ec/FriedrichN15} and more broadly in the area of artificial intelligence where the focus has been on submodular functions and Pareto optimization~\citep{DBLP:conf/nips/QianYZ15,DBLP:conf/nips/QianS0TZ17,DBLP:conf/ijcai/QianSYT17,DBLP:conf/ppsn/NeumannN20,DBLP:conf/gecco/NeumannB021,ROOSTAPOUR2022103597,DBLP:conf/cec/NeumannNQDNVAYWB23,SwbeaOcmn,Sw3Popcc}.

Theoretical studies usually investigate evolutionary algorithms in terms of their runtime and approximation behavior and evaluate the performance of the designed algorithms on classical submodular combinatorial optimization problems.
We refer to the recent book of~\citet{DBLP:books/sp/ZhouYQ19} for an overview. 

However, to our knowledge, there is so far no runtime analysis of evolutionary algorithms for submodular optimization with chance constraints and the runtime analysis of evolutionary algorithms for chance-constrained problems has only started recently for very special cases of chance-constrained knapsack problem~\citep{DBLP:conf/foga/0001S19,DBLP:journals/corr/abs-2102-05778,DBLP:conf/ijcai/DoGN023,DBLP:conf/ijcai/0001W22,DBLP:conf/gecco/0001W23}.

\subsection{Our Contribution}

We provide a first runtime analysis by analyzing \gsemo together with multi-objective formulations that use a second objective taking the chance constraint into account. Furthermore, we carry out experimental investigations for different evolutionary multi-objective algorithms and compare them to state of the art greedy approaches.
We push forward the research on evolutionary algorithms for chance-constrained problems by providing a first runtime analysis for submodular functions with chance constraints.
We consider chance constraints where the stochastic weight $W(X)$ of a feasible solution $X$ can violate a given constraint bound $C$ with probability at most $\alpha$, i.e., $\Pr[W(X) > C] \leq \alpha$ holds.
The bi-objective formulations based on tail inequalities such as Chernoff bounds and Chebyshev's inequality are motivated by recent experimental studies of evolutionary algorithms for the knapsack problem with chance constraints~\citep{DBLP:conf/gecco/XieHAN019,DBLP:conf/gecco/XieN020,DBLP:conf/ecai/AssimiHXN020}.

We analyze \gsemo in the chance-constrained submodular optimization setting investigated in~\citep{DBLP:journals/corr/abs-1911-11451} in the context of the (generalized) greedy algorithm~(\gga). Our analyzes show that the \gsemo is able to achieve the same approximation guarantee as GGA in expected polynomial time for uniform IID weights and the same approximation quality in expected pseudo-polynomial time for independent uniformly distributed weights having the same dispersion. We also point out situations where the use of tail inequalities may prevent \gsemo from obtaining high-quality solutions even for very simple chance-constrained submodular problems.

Furthermore, we study some prominent evolutionary multi-objective algorithms experimentally on the influence maximization problem in social networks and the maximum coverage problem considering two real-word problems. We use the multi-objective problem formulation in a standard setting of the well established algorithms \nsga~\citep{DBLP:journals/tec/DebAPM02} and \spea 
~\citep{zitzler2001spea2}. Our results for uniform constraints show that \gsemo and \nsga significantly outperform the greedy approach~\citep{DBLP:journals/corr/abs-1911-11451} for the considered chance-constrained submodular optimization problems. In the case of the influence maximization problem, the solution obtained by \gsemo has significantly better performance than \nsga for a high budget. 
Similar to the previous investigation in the case of the maximum coverage problem, \gsemo statistically outperforms the greedy approach for most of the settings. Moreover, we consider \gga, \gsemo, \nsga and \spea for settings where the expected weights depend on the degree of the node of the graph for both problems.

Overall, \gga is outperformed by all evolutionary multi-objective algorithms.  We observe that \gsemo is outperformed by \nsga and \spea in most of our experimental settings, but usually achieves better results than the generalized greedy algorithm. 
Furthermore, our experimental investigations show that \spea has a clear advantage for the maximum coverage problem when the cost of selecting a set is dependent on the number of items it contains.

We now describe how this article extends its conference version~\citep{DBLP:conf/ppsn/NeumannN20}. 
The theoretical investigations have been extended by Section~\ref{sec4-extend}.
It provides an analysis of \gsemo for including the empty set into the population when using tail inequalities in the Pareto optimization approaches. The empty set is of particular interest as it enables a greedy behavior of the Pareto optimization approach. However, we also point out situations where Chebyshev's inequality may prevent the algorithm from obtaining high-quality solutions in Section~\ref{sec:cheb}.
Through Section~\ref{sec7}, we significantly extend our experimental investigations  and consider real-world problems with degree based chance constraints for all introduced algorithms. Specifically, we investigate the performance of \spea in addition to the algorithms investigated in the conference version and show that this algorithm performs very well for the degree-based chance constraint setting considered in this article. 

The paper is structured as follows. In Section~\ref{sec2}, we introduce the problem of optimizing submodular functions with chance constraints, its Pareto optimization formulation, and the algorithms \gsemo and \gga. In Section~\ref{sec3},~\ref{sec4-extend} and~\ref{sec4}, we provide runtime analyzes for chance-constrained monotone submodular functions where the weights of the constraints are either identically uniformly distributed or are uniformly distributed and have the same dispersion. In Section~\ref{sec5}, we describe the setting for our experimental investigations.
We carry out experimental investigations that compare the performance of greedy algorithms, \gsemo, \nsga and \spea in Section~\ref{sec6} and~\ref{sec7}, and finish with some concluding remarks.

\section{Preliminaries}
\label{sec2}

Given a set $V = \{v_1, \ldots, v_n\}$, we consider the optimization of a monotone submodular function $f \colon 2^V \rightarrow \R_{\ge 0}$. We call a function monotone iff for every $X, Y \subseteq V$ with $X \subseteq Y$, $f(X) \leq f(Y)$ holds. We call a function $f$ submodular iff for every $X, Y \subseteq V$ with $X \subseteq Y$ and $v \not \in Y$ we have
$$f(X \cup \{v\}) - f(X) \geq f(Y \cup \{v\}) - f(Y).$$

Here, we consider the optimization of a monotone submodular function $f$ subject to a chance constraint where each element $v \in V$ takes on a random weight $w(v)$. Precisely, we examine constraints of the type
\begin{equation}
\label{eq:chance}
    \Pr[W(X) > C] \leq \alpha,
\end{equation}
where $W(X)=\sum_{v\in X}{w(v)}$ is the sum of the random weights of the elements and $C>0$ is the given constraint bound. In this case the parameter $\alpha$ specifies the probability of exceeding the bound $C$ that can be accepted for a feasible solution $X$ and we assume $\alpha < 1/2$ throughout this paper. 

The two settings, we investigate in this paper assume that the weight of an element $v\in V$ is $w(v) \in [a(v)-\delta, a(v)+\delta]$, $a(v) \geq 0$, 
 $\delta \leq \min_{v \in V} a(v)$, is chosen uniformly at random. Here $a(v)$ denotes the expected weight of items $v$.  For our investigations, we use the same setting as in~\citet{DBLP:journals/corr/abs-1911-11451} and assume that each item has the same dispersion $\delta$, and that $\delta$ and $\alpha$ are non-negative constants. 
We call a feasible solution $X$ a $\gamma$-approximation, $0 \leq \gamma \leq 1$, iff $f(X) \geq \gamma \cdot f(OPT)$ where $OPT$ is an optimal solution for the given problem.

\subsection{Chance Constraint Evaluation Based on Tail Inequalities}
\label{sec:tails}
For uniformly distributed weights, the exact joint distribution can be computed as a  convolution if the random variables are independent. Furthermore, there is also an exact expression for the Irwin-Hall distribution~\citep{MR1326603} which assumes that all random variables are independent and uniformly distributed within [0, 1]. However, these approaches are not practical when the number of random components is large and as in our case grows with the input size.
As the probability $\Pr(W(X)> C)$ used in the objective functions is computationally expensive to determine exactly, we use the approach taken in~\citet{DBLP:conf/gecco/XieHAN019} and compute an upper bound on this probability using tail inequalities~\citep{MotwaniRaghavan}. 
We assume that $w(v) \in [a(v)-\delta, a(v)+\delta]$ is uniformly distributed within the given interval for each $v \in V$. This allows us to use Chebyshev's inequality and Chernoff bounds. 

Let $E_W(X) = \sum_{v\in X} a(v)$ denote the expected weight of a given  solution $X$.
The approach based on (one-sided) Chebyshev's inequality used in \citet{DBLP:conf/gecco/XieHAN019} upper bounds the probability of a constraint violation by
  
\begin{equation}\label{thm:CHB}
  \Pr(W(X) > C)\leq\frac{\delta^2|X|}{\delta^2 |X| + 3(C-E_W(X))^2}=:U_{Cheb}(W(X) > C) .
\end{equation}

The approach based on Chernoff bounds used in ~\citet{DBLP:conf/gecco/XieHAN019} upper bounds the probability of a constraint violation by

   \begin{equation}\label{thm:CHF}
   \Pr[W(X)> C] 
\leq \left(\frac{e^{\frac{C-E_W(X)}{\delta |X|}}}{\left(\frac{\delta |X| + C-E_W(X)}{\delta |X|}\right)^{\frac{\delta |X| + C-E_W(X)}{\delta |X|}}}\right)^{\frac{1}{2}|X|}=: U_{Cher}(W(X) > C).
  \end{equation}
  
We will use $U_{Cheb}(W(X) > C)$ or $U_{Cher}(W(X) > C)$ instead of $\Pr(W(X) > C)$ in the multi-objective problem formulation. The general notion $U(W(X) > C)$ refers to the use of an upper bound on the probability of violating the chance constraint. We will examine the use of $U_{Cheb}(W(X) > C)$ and $U_{Cher}(W(X) > C)$ from a theoretical and experimental perspective.
As $Pr(W(X) > C) \leq U(W(X) > C)$ in the settings we consider, $X$ is feasible if $U(W(X) > C) \leq \alpha$ holds.

\subsection{Multi-Objective Formulation}
Following the approach given in~\citet{DBLP:conf/gecco/XieHAN019} for the chance-constrained knapsack problem, we evaluate a set $X$ by the multi-objective fitness function $g(X)=(g_1(X), g_2(X))$ where $g_1$ measures the tightness in terms of the constraint and $g_2$ measures the quality of $X$ in terms of the given submodular function $f$.

We define
\begin{equation}
g_1(X)=\left\{
\begin{array}{lccr}
E_W (X)-C & \text{if} & (C-E_W (X))/(\delta \cdot |X|) \geq 1\\
U(W(X) > C) &\text{if} & {(E_W (X)<C) \wedge ((C-E_W (X))/(\delta |X|) < 1)}\\
1+(E_W (X)-C)  &\text{if} & {E_W (X)\geq C}
\label{g1x}
\end{array} \right.
\end{equation}

and

\begin{equation}
g_2(X)=\left\{
\begin{array}{lccr}
f(X)  & \text{if}& { g_1 (X)\leq \alpha}\\
-1 & \text{if}& g_1(X)>\alpha.
\label{g2x}
\end{array} \right.
\end{equation}
 Here $g_1$ evaluates a solution $X$ with respect to the chance constraint and $g_2$ gives the quality in terms of the submodular function $f$. In our multi-objective setup, we aim to minimize $g_1$ and maximize $g_2$ at the same time.
 
 As we have $\alpha<1/2$ and the considered weight distributions are symmetric, $E_W(X) < C$ has to hold for a feasible solution and we have $g_1(X)\geq 1$ for any solution $X$ with $E_W(X) \geq C$. For a solution $X$ with $E_W(X) < C$, 
the term
$(C-E_W (X))/(\delta \cdot |X|) \geq 1$ in $g_1$ implies that the set $X$ of cardinality $|X|$ has probability $0$ of violating the chance constraint due to the upper bound on the intervals, and $(C-E_W (X))/(\delta \cdot |X|) < 1$ implies that the probability of violating the chance constraint is greater than zero. We use tail bounds to upper bound the constraint violation probability in the latter case. We deem a solution $X$ to be feasible iff $g_1(X) \leq \alpha$ holds. We have $g_2(X)=f(X)$ if $X$ is deemed to be feasible and $g_2(X)=-1$ if it is considered infeasible.

We say a solution $Y$ dominates a solution $X$ (denoted by $Y \succcurlyeq X$) iff $g_1(Y)\leq g_1(X) \land g_2(Y) \geq g_2 (X)$. We say that $Y$ strongly dominates $X$ (denoted by $Y \succ X$) iff $Y \succcurlyeq X$ and $g(Y)\not =g(X)$.
The dominance relation also translates to the corresponding search points. Comparing two solutions, the objective function guarantees that a feasible solution strongly dominates every infeasible solution. The objective function ${g_1}$ ensures that the search process is guided towards feasible solutions and that trade-offs in terms of the probability of a constraint violation and the function value of the submodular function $f$ are computed for feasible solutions.

\subsection{Global SEMO} 

Our multi-objective approach is based on a simple multi-objective evolutionary algorithm called Global Simple Evolutionary Multi-Objective Optimizer (GSEMO, see Algorithm~\ref{alg:GSEMO})~\citep{Giel2003}.
The algorithm encodes sets as bitstrings of length $n$ and the set $X$ corresponding to a search point $x$ is given as $X=\{v_i \mid x_i=1\}$. We use $x$ when referring to the search point in the algorithm and $X$ when referring to the set of selected elements and use applicable fitness measure for both notations in an interchangeable way.
GSEMO starts with a random search point $x\in\{0,1\}^n$. In each iteration, an individual $x \in P$ is chosen uniformly at random from the current population $P$.
In the mutation step, it flips each bit with a probability $1/n$ to produce an offspring $y$.
$y$ is added to the population if it is not strongly dominated by any other search point in $P$. If $y$ is added to the population, all search points dominated by $y$ are removed from the population $P$.

We analyze \gsemo in terms of its runtime behavior to obtain a good approximation. The expected time of the algorithm required to achieve a given goal is measured in terms of the number of iterations of the repeat loop until a feasible solution with the desired approximation quality has been produced for the first time.

\subsection{Generalized Greedy Algorithm}
For comparing the results of our Pareto optimization approaches, we consider the generalized greedy algorithm (GGA) shown in Algorithm~\ref{alg:GGA}. GGA has been already used frequently in the deterministic setting~\citep{DBLP:journals/ipl/KhullerMN99,DBLP:conf/kdd/LeskovecKGFVG07} and recently for the chance-constrained setting~\citep{DBLP:journals/corr/abs-1911-11451} that we investigate in this article. 
The algorithm starts with an empty set, and in each iteration, adds an element based on a ratio. This ratio considers the additional gain with respect to the submodular function $f$. It ensures that the expected weight increase $E[W(X \cup\{v\})-W(X)]$ of the constraint is maximal, while satisfying the chance constraint. 
In order to determine whether an element $v^*$ can be added an upper bound $U(W(X\cup \{v^*\})> C)$ on the probablity of the constraint violation (using Chebyshev's inequality or Chernoff bound) is used (see Section~\ref{sec:tails}).
The algorithm terminates if no further element can be added. The algorithm compares the constructed greedy solution with the best feasible solution containing exactly a single element and returns the best one of these two solutions.

If all weights or expected gains in weights are the same for all elements then the GGA is equivalent to the standard greedy algorithm (GA) which selects the element with the largest marginal gain at each step. This approach has also been examined in~\citet{DBLP:journals/corr/abs-1911-11451} for the case of chance constraints with uniform IID weights. The obtained solution is automatically at least as good as any feasible solutions containing a single element when considering monotone submodular functions as done in this article.

\begin{algorithm}[tp]
 Choose $x \in \{0,1\}^n$ uniformly at random\;
 $P\leftarrow \{x\}$\;

\Repeat{$stop$}{
Choose $x\in P$ uniformly at random\;
Create $y$ by flipping each bit $x_{i}$ of $x$ with probability $\frac{1}{n}$\;
\If{$\not\exists w \in P: w \succ y$} {
  $P \leftarrow (P \cup \{y\})\backslash \{z\in P \mid y \succcurlyeq z\};$}
    }
\caption{Global SEMO} \label{alg:GSEMO}
\end{algorithm}

\begin{algorithm}[t]
	\SetKwInOut{Input}{input}
    \Input{
    Set of elements $V$, budget constraint $C$, failure probability $\alpha$.}
    $X \leftarrow\emptyset$\;
		$V^\prime \leftarrow V$\;
\Repeat{$V^\prime = \emptyset$}{$v^*\leftarrow \argmax_{v\in V^\prime}\frac{f(X \cup\{v\})-f(X)}{E[W(X\cup \{v\})-W(X)]\label{line:KostenNutzen}}$\;
    \If{$U(W(X\cup \{v^*\})> C)\leq \alpha$} 
    {$X\leftarrow X\cup \{v^*\}$\;
    }
    $V^\prime \leftarrow V^\prime\setminus \{v^*\}$\;}
    $v^* \leftarrow \argmax_{\{v\in V;\Pr[W(v)> C] \leq \alpha\} }f(v)$\;
    \Return{$\argmax_{Y\in \{X,\{v^*\}\}}f(Y)$}\;
    \caption{Generalized Greedy Algorithm (GGA)}\label{alg:GGA}
    \end{algorithm}

\section{Runtime Analysis for Uniform IID Weights}
\label{sec3}

We now provide a runtime analysis of \gsemo which shows that the algorithm is able to obtain a good approximation for important settings where the weights of the constraint are chosen according to a uniform distribution with the same dispersion.

We first investigate the case of uniform identically distributed (IID) weights. Here each weight is chosen uniformly at random in the interval $[a-\delta, a+\delta]$, $\delta \leq a$. The parameter $\delta$ is called the dispersion and models the uncertainty of the weight of the items. Given the constraint bound $C$ and the expected cost of an item $a$, there are at most 
$k=\min\{n, \lfloor C/a \rfloor \}$ items that can be chosen without violating the chance constraint (Equation~\ref{eq:chance}) if $\alpha \leq 1/2$ holds as we assume throughout this article.
\begin{theorem}
\label{thm:iid}
Let $k=\min\{n+1, \lfloor C/a \rfloor +1\}$ and assume $\lfloor C/a \rfloor=\omega(1)$.
Then the expected time until \gsemo using objective function $g$ has computed a $(1-o(1))(1-1/e)$-approximation for a given monotone submodular function under a chance constraint with uniform iid weights is $O(nk(k+\log n))$.
\end{theorem}

\begin{proof}
Every item has expected weight $a$ and dispersion $\delta$. This implies $g_1(X)=g_1(Y)$ iff $|X|=|Y|$ and $E_W(X)=E_W(Y)<C$. As \gsemo only stores for each fixed $g_1$-value one single solution, the number of solutions with expected weight less than $C$ is at most $k=\min \{\lfloor C/a \rfloor +1,n+1\}$.
Furthermore, there is at most one infeasible individual $X$ in the population which has $g_2(X)=-1$ and such an infeasible individual is dominated by any feasible solution. Hence, the maximum population that \gsemo encounters during the run of the algorithm is at most $k$.

We first consider the time until \gsemo has produced the bitstring $0^n$.
This is the best individual with respect to $g_1$ and once included will always stay in the population. The function $g_1$ is strictly monotonically increasing with the size of the solution. Hence, selecting the individual in the population with the smallest number of elements and removing one of them leads to a solution with fewer elements and therefore with a smaller $g_1$-value. Let $\ell=|x|_1$ be the number of elements of the solution $x$ with the smallest number of elements in $P$. Then flipping one of the $1$-bits corresponding to these elements reduces $\ell$ by one and happens with probability at least $\ell/(en)$ once $x$ is selected for mutation. The probability of selecting $x$ is at least $1/k$ as there are at most $k$ individuals in the population. Using the methods of fitness-based partitions, the expected time to obtain the solution $0^n$ is at most

$$ \sum_{\ell=1}^n \left(\frac{\ell}{ekn}\right)^{-1} = O(nk \log n).
$$

Let $\kopt=\lfloor C/a \rfloor$, the maximal number of elements that can be included in the deterministic version of the problem.

The function $g_1$ is strictly monotonically increasing with the number of elements and each solution with same number of elements has the same $g_1$-value.

We consider the solution $X$ with the largest value of $r$ for which $|X| \leq r$

\[
f(X) \geq (1 - (1- 1/\kopt)^r) \cdot f(OPT)
\]
holds in the population and the mutation which adds an element $x$ with the largest marginal increase $g_2(X \cup \{x\}) - g_2(X)$ to $X$. The probability for such a step picking $X$ and carrying the mutation with the largest marginal gain is $\Omega(1/kn)$ and its waiting time is $O(kn)$.

This leads to a solution $Y$ with $|Y| \leq r+1$ elements for which 

\[
f(X) \geq (1 - (1- 1/\kopt)^{r+1}) \cdot f(OPT)
\]
holds.
The maximal number of times such a step is required after having included the search point $0^n$ into the population is $k$ which gives the runtime bound of $O(k^2n)$.

For the statement on the approximation quality, we make use of the lower bound on the maximal number of elements that can be included using the Chernoff bound and Chebyshev's inequality used in Theorem 3 of~\citet{DBLP:journals/corr/abs-1911-11451}. 

Using Chebyshev's inequality (Equation~\ref{thm:CHB}) at least 
$$    
k_1^*= \max \left\{r \mid  r+ \frac{\sqrt{(1-\alpha)r\delta^2}}{\sqrt{3\alpha}a}   \leq \kopt \right\}
$$
elements can be included and when using Chernoff bound (Equation~\ref{thm:CHF}), at least
$$
    k_2^*= \max \left\{r \,\middle|\,  r+ \frac{\sqrt{3\delta r \ln(1/\alpha)}}{a}   \leq \kopt \right\} \label{kstar}
$$
elements can be included.

Including $k^*$ elements in this way leads to a solution $X^*$ with
$$
f(X^*)\geq (1 - (1- 1/\kopt)^{k^*}) \cdot f(OPT).
$$
As shown in Theorem 3 in~\citet{DBLP:journals/corr/abs-1911-11451}, both values of $k_1^*$ and $k_2^*$ yield $(1-o(1))(1-1/e) \cdot f(OPT)$ if $\lfloor C/a \rfloor=\omega(1)$ which completes the proof.
\end{proof}

\section{Runtime Analysis for Uniformly Distributed Weights with the Same Dispersion Using Tail Inequalities}
\label{sec4-extend}
We now assume that the expected weights do not have to be the same, but still require the same dispersion for all elements, i.e., $w(v) \in [a(v)-\delta, a(v)+\delta]$ holds for all $v \in V$. 

\subsection{Analysis of computing feasible solutions using surrogate functions based on tail inequalities}

First, we investigate the fitness function $g$ for uniformly distributed weights with the same dispersion. We investigate the use of Chebyshev's inequality and the Chernoff bound and examine the time until \gsemo has produced a feasible solution and subsequently the search point $0^n$ for the first time. Obtaining the search point $0^n$ is crucial for the algorithm in order to gain the ability of making use of the Pareto optimization setup. The analysis here is of particular interest from a theoretical perspective as it requires to analyze the progress with respect to the surrogate function given by Chebyshev's inequality and the Chernoff bound.
Let $a_{\max}=\max_{v \in V} a(v)$ and $a_{\min}=\min_{v \in V} a(v)$, and $\delta \leq a_{\min}$. By studying the optimization process when using Chebyshev's inequality and the Chernoff bound as surrogate functions to estimate the probability of a constraint violation, we get the following upper bound on the expected time until \gsemo has produced a feasible solution for the first time.

\begin{lemma}
\label{lem:feasible}
Starting with an arbitrary solution, the expected time until \gsemo has produced a feasible solution when working with the objective function $g$ is $O(n^2 (\log n + \log(a_{\max}/a_{\min})))$.
\end{lemma}

\begin{proof}
As long as a feasible solution has not been obtained, the population size of \gsemo is $1$ as each infeasible solution $x$ has $g_2(x)=-1$ and the population consists of the infeasible solution with the smallest $g_1$-value obtained so far.
We first consider the time to reach a solution $x$ with $E_W(X)<C$. To do so, we use multiplicative drift analysis~(Theorem 3 in~\citet{DBLP:journals/algorithmica/DoerrJW12}) with the drift function $E_W(X)$.
Let $x$ be the current solution with $E_W(X)\geq C$, then the expected weight of the offspring is at most
$$(1-1/(en))\cdot E_W(X)$$
which is obtained by considering the average decrease of the weight obtained by all single $1$-bit flips.
Using multiplicative drift analysis~(Theorem~ 3 in~\citet{DBLP:journals/algorithmica/DoerrJW12}) with $s_{\max} = n a_{\max}$ and $s_{\min} = a_{\min}$, a solution $x$ with $E_W(X)<C$  has been obtained when using $g$ after an expected number of $O(n (\log n + \log(a_{\max}/a_{\min})))$ steps.

We consider the case where we have obtained a solution $x$ with $E_W(X)<C$ and $U(W(X) > C)>\alpha$, and analyze the time to produce a feasible solution, i.e., a solution that violates the chance constraint with probability at most $\alpha$.
Let $\beta(x)= U(W(X) > C)$ be an upper bound (obtained by Chebyshev's inequality or Chernoff bound) on the probability that solution $x$ violates the chance constraint.
Furthermore, let $k$, $1 \leq k \leq n$, be a fixed number of ones and let $\beta_k = \min \{\beta(x) \mid |x|_1=k\}$ be the smallest probability of violating the chance constraint among all solutions with $k$ elements. Note that $\beta_0 \leq \ldots \leq \beta_n$ and we set $\beta_0=0$ as the solution not containing any elements has a zero probability of violating the chance constraint.
We show that \gsemo spends an expected number of $O(n (\log n + \log a_{\max}))$ iterations on creating offspring of solutions with $k$ elements before producing a solution with at most $k-1$ elements whose probability of violating the chance constraint is less then $\beta_k$. This implies that solutions with $k$ elements are not accepted anymore until a feasible solution has been produced for the first time.

Let $x$ and $y$ be two (infeasible) solutions with $|x|_1=|y|_1=k$. Then $\beta(x) < \beta(y)$ iff $E_W(X) < E_W(Y)$. This holds for both Chebyshev's inequality (Equation~\ref{thm:CHB}) and the Chernoff bound (Equation~\ref{thm:CHF}).
We consider the expected number of offspring generated from a solution $x$ with $|x|_1=k$ until a solution $\hat{x}$ with $\beta(\hat{x}) < \beta_k$ has been obtained. Once this has occurred, no infeasible solution with at least $k$ elements is accepted.
We measure the progress of the algorithm towards a solution with constraint violation probability $\beta_k$ in terms of the expected weight of the solution which needs to be minimized.
Let $x$ be the current infeasible solution with $|x|_1=k$. Then flipping a 1-bit $x_i$, leads to a solution $y$ with $|y|_1=k-1$ and $E_W(Y) = E_W(X)- a_i$. Furthermore, we have  $E_W(Y')< E_W(Y)$ for any solution $y'$ for which $|y'|_1=k$ and $\beta(y') \leq \beta(y)$ holds. 
This implies that the next accepted solution with $k$ elements has expected weight at most $E_W(Y)$. Considering all operations flipping a single $1$-bit the average decrease is $E_W(X)/k$ and a single $1$-bit flip happens with probability at least $k/(en)$.
This implies that the next solution $x'$ with $|x'|_1=k$ that is accepted by the algorithm fulfills
$$E_W(X') \leq (1-1/(en)) \cdot E_W(X).$$
Using again multiplicative drift analysis (Theorem~ 3 in~\citet{DBLP:journals/algorithmica/DoerrJW12}) with $s_{\max} = n a_{\max}$ and $s_{\min} = a_{\min}$, a  solution with $\hat{x}$ with $\beta(\hat{x}) < \beta_k$ has been obtained after creating an expected number of  $O(n (\log n + \log(a_{\max}/a_{\min})))$ offspring of a solution with $k$ bits. 
Considering all values of $k$, $1 \leq k \leq n$, the expected number of iterations until a feasible solution has been produced for the first time is $O(n^2 (\log n + \log(a_{\max}/a_{\min})))$.
\end{proof}

We now further investigate how solutions with an even smaller violation probability may be constructed. Trade-off solutions with respect to the given objective function and the degree of constraint violation are generally beneficial in Pareto optimization approaches as they enable a greedy behavior of the algorithm.
The maximum population size $P_{\max}$ that the \gsemo encounters during the run of the algorithm is a crucial parameter when analyzing the runtime of the algorithm~\citep{DBLP:conf/ijcai/QianSYT17,ROOSTAPOUR2022103597} as the parent producing an offspring is chosen uniformly at random from the population. The previous analysis considered the time to produce a feasible solution for the first time. Until then the population size of \gsemo is $1$ as each infeasible solution obtains a $g_2$-value of $-1$. We now consider the process after a feasible solution has been obtained for the first time which results in the algorithm creating trade-offs according to the two objective functions $g_1$ and $g_2$.

\begin{lemma}
\label{lem:zeroprob}
Let $P_{\max}$ be the maximum population size that \gsemo encounters during the run of the algorithm.
Then the expected time until \gsemo using the fitness function $g$ has produced a solution $x$ that violates the constraint $C$ with probability zero is $O(P_{\max}n^2 (\log n + \log(a_{\max}/a_{\min})))$.
\end{lemma}
\begin{proof}
Due to Lemma~\ref{lem:feasible}, the expected time until \gsemo has produced a feasible solution for the first time is $O(n^2 (\log n + \log(a_{\max}/a_{\min})))$. Having produced a feasible solution, the population size may grow over time. 
We always focus on the solution $X$ with the smallest $\beta(x)$-value. We assume that $E_W(X)<C$ and that the solution of minimal $\beta$-value has a positive probability of violating the constraint bound, i.e.,  $(C-E_W(X))/(\delta |X|) <1$ holds.

We can reuse the arguments in the last part of the proof of Lemma~\ref{lem:feasible}.
We consider again the solution $x$ with the smallest $\beta(x)$ value in the population and note that this value is positive as long as a solution violating the constraint with zero probability has not been obtained. 
The solution $X$ with minimal $\beta(x)$ is selected with probability $1/P_{\max}$ in the next iteration and a solution with $(C-E_W(X))/(\delta |X|) \geq 1$, which has a probability zero of violating the constraint bound $C$, is produced in expected time $O(P_{\max}n^2 (\log n + \log(a_{\max}/a_{\min})))$.
\end{proof}

Finally, we bound the time until \gsemo has included the search point $0^n$ into the population which is the crucial search point for enabling the greedy behavior.

\begin{theorem}
\label{lem:zerostring} 
The expected time until \gsemo using the fitness function $g$ has included the search point $0^n$ in the population  is $O(P_{\max}n^2 (\log n + \log(a_{\max}/a_{\min})))$.
\end{theorem}

\begin{proof}
We work under the assumption that for $g$ a solution has been obtained that violates the constraint bound $C$ with zero probability. For solutions that violate the constraint bound $C$ with probability zero, the expected cost should be minimizing for objective $g_1$. 
We always consider the individual $x$ with the smallest expected cost $E_W(X)$ in the population.
Flipping an arbitrary $1$-bit of $x$ leads to an individual with a smallest $g_1$-value and is therefore accepted. Furthermore, the total weight decrease of these $1$-bit flips is $E_W(X)$ which also equals the total weight decrease of all single bit flip mutation when taking into account that $0$-bit flips give decrease of the expected weight of zero. A mutation carrying out a single bit flip happens each iteration with probability at least $1/e$.
The expected decrease in $g_1$ is therefore at least by a factor of $(1- 1/(P_{\max}en))$ and the expected minimal $g_1$-value of the population in the next generation is at most
$$ (1- 1/(P_{\max}en)) \cdot E_W(X).
$$
We use again multiplicative drift analysis (Theorem~ 3 in~\citet{DBLP:journals/algorithmica/DoerrJW12}), to upper bound the expected time until the search point $0^n$ is included in the population.
As $a_{\min}\leq a(x) \leq n a_{\max}$ holds for any search point $x\not=0^n$, the search point $0^n$ is included in the population after an expected number of $O(P_{\max}n(\log n +\log (a_{\max}/a_{\min})))$ steps. Taking into account the expected time to produce a feasible solution given in Lemma~\ref{lem:feasible} and the time for producing a solution which violates the constraint bound $C$ with probability zero for $g$ given in Lemma~\ref{lem:zeroprob}, the expected time to include the search point $0^n$ in the population is $O(P_{\max}n^2 (\log n + \log(a_{\max}/a_{\min})))$. 
\end{proof}

The previous analysis shows that \gsemo using the fitness function $g$ is able to produce the search point $0^n$ which is crucial to enable the greedy behavior of the algorithm. However, the objective $g_1$ is non-linear due to the tail inequalities involved. This prevents us from obtaining good approximation guarantees for \gsemo when using the fitness function $g$. We investigate this issue for the use of Chebyshev's inequality in the following.

\subsection{Example where Chebyshev's inequality may prevent progress}
\label{sec:cheb}
We now present a non-monotone submodular example problem where beneficial intermediate solutions in order to obtain high-quality solutions are rejected by \gsemo when using Chebyshev's inequality.

\subsubsection{Preventing optimal solutions}
\label{sec:localopt}

The instance $I$ consists of $n$ elements $v_i$, $1 \leq i \leq n$.
The submodular function $f$ is defined as
$$
f_I(x) = \sum_{i=1}^n f(v_i) x_i
$$
where $f(v_i)$ is a contribution to the overall function value potentially dependent on other elements selected.

We have $f(v_i)=1$, $1 \leq i \leq n-1$ if $v_n$ is not present in $x$ and $f(v_i)=0$, $1 \leq i \leq n-1$, if $v_n$ is present in $x$. 
Furthermore, $f(v_n)=0.6n$ independent of the other elements chosen in $x$. 
Note that $f_I$ is submodular but non-monotone as adding element $v_n$ to any solution with more than $0.6n$ elements chosen among the first $n-1$ elements reduces the function value. When considering the subspace defined by $x_n=0$ for all $x \in \{0,1\}^n$, $f_I$ is equal to the classical benchmark function OneMax frequently considered in the area of runtime analysis~\citep{BookDoeNeu} and therefore monotone and submodular.
We set $C=1.5n$ and $\delta=4$.
We have $a(v_i)=1$, $1 \leq i \leq n-1$, and $a(v_n)=0.6n$. 
Let $\hat{x}$ be the solution consisting of element $v_n$ only.
The optimal solution $x^*$ consists of all elements $v_i$, $1 \leq i \leq n-1$, and we have $f(x^*)=n-1>f(\hat{x})=0.6n$.
Let $\alpha$, $0< \alpha< 1/2$, be an arbitrary constant.
Using Chebyshev's inequality for the upper bound on the constraint violation in the function $\beta$ (see proof of Lemma~\ref{lem:feasible}), we have
$\beta(x^*) = \frac{\delta^2(n-1)}{\delta^2(n-1) + 3(1.5n - (n-1))^2}= \frac{16(n-1)}{16(n-1) + 3(1.5n - (n-1))^2} = O(1/n) < \alpha$ which implies that $x^*$ is feasible.

We show that if the solution $\hat{x}$ consisting of element $v_{n}$ only is present
in the population, then no solution with $k$, $n/2 \leq k < 0.6n$ elements is accepted. Note that for a solution $x$ with $|X|\geq n/2$ and $E_W (X)<C$, we have $E_W(X) \geq n/2$  and therefore $(C-E_W (X))/(\delta |X|) \leq n/(4 \cdot n/2) < 1$ which implies that it is using a tail inequality to get an upper bound on $\Pr(w(x)>C)$ using the fitness function $g$. 
For $\hat{x}$, the value $E_W(\hat{X}) -C$ would be used as part of $g_1$ as ${(E_W(X)<C) \wedge ((C-E_W (X))/(\delta |X|)= 0.9n/4 > 1)}$.
In the following, we assume that $\hat{x}$ is evaluated using Chebyshev's inequality to illustrate the effect of having a large variance preventing the algorithm from progressing. Chebyshev's inequality has been used for such cases in~\citet{DBLP:conf/gecco/XieHAN019} and it is important to understand the implications of using such surrogate functions.
Note that we have $E_W(\hat{X})- C < \beta(\hat{x})$ and all following results also hold when applying the original function $g$ to $\hat{x}$.

Each solution with less than $0.6n$ elements not containing $v_{n}$ has fitness less than $0.6n$.  
We now show that each solution $x$ with $k$, $n/2 \leq k < 0.6n$, elements is also worse than $\hat{x}$ when evaluating it with respect to Chebyshev's inequality.

We have
$$\beta(\hat{x}) = \frac{\delta^2}{\delta^2 + 3 (0.9n)^2}= \frac{\delta^2}{\delta^2 + 2.43n^2}$$

$$\beta(x) \geq \frac{\delta^2k}{\delta^2k + 3 n^2}$$

and 

\begin{eqnarray*}
& & \beta(\hat{x}) < \beta(x)\\
& \Longleftrightarrow & \frac{\delta^2}{\delta^2 + 2.43n^2} < \frac{\delta^2k}{\delta^2k + 3 n^2}\\
& \Longleftrightarrow & \frac{\delta^2 \cdot (\delta^2k + 3 n^2) - \delta^2 k \cdot (\delta^2 + 2.43n^2)}{(\delta^2 + 2.43n^2)\cdot (\delta^2k + 3 n^2)} < 0\\
& \Longleftrightarrow & \delta^2 \cdot (\delta^2k + 3 n^2) < \delta^2 k \cdot (\delta^2 + 2.43n^2)\\
& \Longleftrightarrow &  3 \delta^2 n^2 < 2.43 \delta^2 k n^2\\
& \Longleftrightarrow &  (3/2.43) < k\\
\end{eqnarray*}
which holds for any $k\geq 2$.

Hence, once the solution $\hat{x}$ consisting of element $v_{n}$ only is included in the population, no solution with at least $n/2$ and less than $0.6n$ elements is accepted. Furthermore, no solution $z$ containing element $v_n$ and some additional elements is accepted as we have $f(z)=f(\hat{x})$ and $g_1(z) > g_1(\hat{x})$. Mutation steps flipping at least $0.1n$ elements, occur in each mutation step with probability at most $e^{-\Omega(n)}$ which implies an exponential optimization time for \gsemo using the fitness function $g$ when having included the solution $\hat{x}$ into the population before having produced a solution consisting of more than $0.6n$ of the first $n-1$ elements.
In fact, \gsemo is already prevented from obtaining a solution with $f_I(x)> 0.6n$ when the initial solution contains element $v_n$. This is summarized in the following theorem.
\begin{theorem}
The expected time until \gsemo has produced a solution $x$ with $f_I(x) > 0.6n$ for instance $I$ is $e^{\Omega(n)}$.
\end{theorem}

\begin{proof}
The initial solution contains the element $v_n$ with probability $1/2$ and has at most $0.51n$ elements with probability $1-e^{-\Omega(n)}$ using Chernoff bounds.  In this case, no other solution with $k$, $0.51n < k < 0.6n$, elements is accepted. In order to produce an accepted solution with at least $0.6n$ elements $\Theta(n)$ elements have to flip in a single mutation steps which happens with probability $e^{-\Omega(n)}$. Hence the optimization time is $e^{\Omega(n)}$ with probability $1/2-o(1)$ and therefore the expected time until \gsemo has produced a solution of fitness greater than $0.6n$ is $e^{\Omega(n)}$.
\end{proof}

On the other hand, the problem observed here for \gsemo can be fixed by minimizing the expected cost of a solution as the first objective. Then the solution containing the element $v_{n}$ is no longer dominating solutions consisting of $i \in [0.5n, 0.6n[$  of the first $n-1$ elements. 
Using the expected cost as first objective, all solutions containing any subset of the first $n-1$ elements is Pareto optimal.

\subsubsection{Preventing good approximations}

We now consider an example that further points out potential problems when using Chebyshev's inequality. We assume that each solution $x$ with $E_W(X) < C$ is evaluated using Chebyshev's inequality. This approach has been taken by~\citet{DBLP:conf/gecco/XieHAN019} and seems to be natural when dealing with stochastic values that are not limited to specific intervals, e.g. for components that are chosen according to a normal distribution.

We consider an example where the approximation ratio of \gsemo using Chebyshev's inequality for solutions $x$ with $E_W(X) < C$ may be arbitrarily bad when starting with the search point $0^n$ as the initial solution. 

We consider instance $I'$ defined by the function $f_{I'}$ which is submodular but non-monotone and follows the structure of $f_I$ introduced in Section~\ref{sec:localopt}.
We have
$$
f_{I'}(x) = \sum_{i=1}^{n} f(v_i) x_i
$$
where $f(v_i)$ is a contribution to the overall function value potentially dependent on other elements selected.
For the elements $v_i$, $1 \leq  i \leq n-1$, we have $a(v_i)=1$, $f(v_i)=1$ if $v_{n}$ is not present and $f(v_i)=0$ if $v_{n}$ is  present in $x$. 
Furthermore, we have $a(v_{n})= n^{\epsilon}$, $f(v_{n})= n^{\epsilon}$ where $0<\epsilon<1$ is a constant. We set $C=n$.

Let $\hat{x}$ be the solution containing element $v_n$ only and $x$ be a solution containing $k \geq n^{\epsilon/2}$ elements other than $v_n$.

Using Chebyshev's inequality, we have
$$\beta(\hat{x}) = \frac{\delta^2}{\delta^2 + 3 (n - n^{\epsilon})^2}$$

$$\beta(x) \geq \frac{\delta^2k}{\delta^2k + 3 (n-n^{\epsilon/2})^2}$$

and 

\begin{eqnarray*}
& & \beta(\hat{x}) < \beta(x)\\
& \Longleftrightarrow & \frac{\delta^2}{\delta^2 + 3 (n - n^{\epsilon})^2} < \frac{\delta^2k}{\delta^2k + 3 (n-n^{\epsilon/2})^2}\\
& \Longleftrightarrow & \frac{\delta^2 \cdot (\delta^2k + 3 (n-n^{\epsilon/2})^2) - \delta^2k \cdot (\delta^2 + 3 (n - n^{\epsilon})^2)}{(\delta^2 + 3 (n - n^{\epsilon})^2) \cdot (\delta^2k + 3 (n-n^{\epsilon/2})^2)} < 0\\
& \Longleftrightarrow & \delta^2 \cdot (\delta^2k + 3 (n-n^{\epsilon/2})^2) < \delta^2k \cdot (\delta^2 + 3 (n - n^{\epsilon})^2)\\
& \Longleftrightarrow &  3\delta^2 (n-n^{\epsilon/2})^2 <  3 \delta^2k (n - n^{\epsilon})^2\\
& \Longleftrightarrow &  (n-n^{\epsilon/2})^2/(n - n^{\epsilon})^2 <k \\
\end{eqnarray*}

We have $(n-n^{\epsilon/2})^2 <n^2$ and $(n - n^{\epsilon})^2\geq n^2/2$ which implies that the last inequality holds for $k> n^2/(n^2/2)=2$. Hence, each solution $x$ which does not contain $v_{n}$ and contains at least $n^{\epsilon/2}$ elements and less than $n^{\epsilon}$ elements is not accepted. The optimal solution contains all elements except $v_{n}$ and therefore the approximation quality of any solution containing at least $n^{\epsilon}$ elements is $n^{\epsilon}/n= O(1/n^{1-\epsilon})$. This shows that the approximation quality may be become arbitrarily bad when using Chebyshev's inequality as a surrogate function for evaluating the chance constraint.
We summarize the situation for instance $I'$ in the following theorem.

\begin{theorem}
\label{thm:inapprox}
Having obtained a population which contains the solution $\hat{x}$ and where each solution has less then $n^{\epsilon}$ elements, the time for \gsemo to obtain a solution with approximation quality better than $n^{-1+\epsilon}$ for instance $I'$ is $e^{\Omega({n^{\epsilon}})}$ with probability $1-e^{-\Omega(n^{\epsilon})}$ for any constant $\epsilon$, $0 < \epsilon <1$, when using Chebyshev's inequality for constraint evaluation. 
\end{theorem}
\begin{proof}
As we are interested in the asymptotic behavior, we assume that $n$ is sufficiently large.
If $\hat{x}$ is included in the population then producing a solution from $\hat{x}$ that contains additional elements is not accepted as it increases the expected weight without changing the $f$ value. Furthermore, any solution $x$ with $k$, $n^{\epsilon/2} \leq k \leq n^{\epsilon}-1$, elements is strongly dominated by $\hat{x}$ and therefore removed from the population once $\hat{x}$ is included in the population and such a solution $x$ is not accepted in subsequent steps of the algorithm. A mutation flipping at least $n^{\epsilon}- n^{\epsilon/2}\geq n^{\epsilon} /2$ bits happens with probability at most 

\begin{eqnarray*}
\sum_{i=n^{\epsilon}/2}^n \binom{n}{i}n^{-i} &\leq & \sum_{i=n^{\epsilon}/2}^n \frac{n^i}{i!} \cdot n^{-i} =  \sum_{i=n^{\epsilon}/2}^n \frac{1}{i!}
\leq  \sum_{i=n^{\epsilon}/2}^n \left(\frac{e}{i}\right)^i
\leq \sum_{i=n^{\epsilon}/2}^n \left(\frac{e}{2e^2}\right)^i \\
& \leq & \sum_{i=n^{\epsilon}/2}^n \left(2e\right)^{-i}
\leq n \cdot 2^{-n^{\epsilon}/2} \cdot e^{-n^{\epsilon}/2} 
\leq e^{-n^{\epsilon}/2}
\end{eqnarray*}

where the last inequality holds as $n \cdot 2^{-n^{\epsilon}/2} \leq 1$ for $n^{\epsilon}> 2\log n$.
Considering a phase of $e^{n^{\epsilon}/4}$ steps and using the union bound, the probability that such a step happens within this phase is $e^{n^{\epsilon}/4} \cdot e^{-n^{\epsilon}/2} \leq e^{-n^{\epsilon}/4} = e^{- \Omega(n^{\epsilon})}$.
This implies that the time to produce a solution with at least $n^{\epsilon}$ elements is $e^{\Omega({n^{\epsilon}})}$ with probability $1-e^{-\Omega(n^{\epsilon})}$ for the conditions stated in the theorem. 
\end{proof}

Note that many Pareto optimization approaches choose the search point $0^n$ as the initial solution~\citep{DBLP:conf/nips/QianYZ15,DBLP:conf/nips/QianS0TZ17,DBLP:conf/ijcai/QianSYT17}. This implies that the probability of producing the solution $\hat{x}$ containing element $v_n$ only before producing a solution with at least $n^{\epsilon}$ elements is at least $\Omega(1/n^2)$ as population size of \gsemo for instance $I'$ is $O(n)$ as there are $O(n)$ different fitness values for $f_{I'}$ and the solution $\hat{x}$ is produced from $0^n$ with probability at least $\Omega(1/n)$ by flipping the bit corresponding to $v_n$. If this occurs then we are in the situation of Theorem~\ref{thm:inapprox} and the time to reach an approximation better than $n^{-1+\epsilon}$ is $e^{\Omega({n^{\epsilon}})}$ with probability $1-e^{-\Omega(n^{\epsilon})}$.

In the next section, we show that for uniformly distributed weights with the same dispersion one can use the expected cost as the first objective and obtain good approximation guarantees for \gsemo.

\section{Runtime Analysis for Uniformly Distributed Weights with the Same Dispersion Using Expected Weights as First Objective}
\label{sec4}
As shown in the previous section, GSEMO might be prevented from progressing well when using Chebyshev's inequality as part of the objective function $g$.
We now present a different formulation, that allows to obtain approximation guarantees for the Pareto optimization approach using \gsemo.
We consider the (to be minimized) objective function
$
\hat{g_1}(X)=E_W (X)
$
(instead of $g_1$)
together with the previously defined objective function $g_2$ and evaluate a set $X$ by $\hat{g}(X)=(\hat{g}_1(X), g_2(X))$. 
We have $Y \succeq X$ iff $\hat{g_1}(Y) \leq \hat{g_1}(X)$ and $g_2(Y) \geq g_2(X)$. Note that the tail inequalities are not used in $\hat{g}_1$, but still used in $g_2$ in order to determine whether a solution is determined as feasible.

Recall that $a_{\min} = \min_{v \in V} a(v)$ and $a_{\max} = \max_{v \in V} a(v)$ denote the minimal and maximal expected cost of the given elements.
 The following theorem shows that \gsemo is able to obtain a $(1/2-o(1))(1-1/e)$-approximation if $\omega(1)$ elements can be included in a solution.

\begin{theorem}
\label{thm:samedisp} 
Assume $C/a_{\max}=\omega(1)$ and $a_{\min} = \Omega(1)$.  Then \gsemo using $\hat{g}$ obtains a $(1/2-o(1))(1-1/e)$-approximation for a given monotone submodular function under a chance constraint with uniformly distributed weights having the same dispersion in expected time $O(P_{\max}n  (C/a_{\min}+n + \log (a_{\max}/a_{\min})))$.
\end{theorem}

\begin{proof}

The fitness function $\hat{g}$ uses the expected cost in $g_1$ which should be minimized. We can reuse the analysis from the proof of Lemma~\ref{lem:zerostring} and consider always the individual with the smallest cost and bound the time to obtain the search point $0^n$ by $O(P_{\max}n(\log n +\log (a_{\max}/a_{\min})$ using multiplicative drift analysis as in Lemma~\ref{lem:zerostring}.

After having included the search point $0^n$ in the population, we follow the analysis for subset selection with general deterministic cost constraints carried out in~\citet{DBLP:conf/ijcai/QianSYT17,DBLP:conf/aaai/BianF0020}.

 We denote by  $\hat{g_1}^*$, the maximal $\hat{g_1}$-value for which there is a solution $x \in P$ with  $\hat{g_1}(x) \leq \hat{g_1}^*$ and

$$g_2(X) \geq  \left(1- e^{-\hat{g_1}^*/C} \right) \cdot f(OPT).
$$

We use $\hat{g_1}^*$ to track the progress of the algorithm and it has been shown in~\citet{DBLP:conf/ijcai/QianSYT17}(proof of Theorem 2) that $\hat{g_1}^*$ does not decrease during the optimization process of \gsemo.
Note that $\left(1- e^{-\hat{g_1}^*/C} \right) \cdot f(OPT)$ is monotonically increasing in $\hat{g_1}^*$ and
that we have $\hat{g_1}^*\geq 0$ once the search point $0^n$ has been included in the population as $\hat{g}_1(0^n)=0$ and $g_2(0^n)=0$.

Choosing $x$ for mutation and flipping the $0$-bit of $x$ corresponding to the largest marginal gain in terms of $g_2/\hat{g_1}$ gives a solution $y$ for which

\begin{eqnarray*}
g_2(Y) & \geq &  f(X) + \frac{\hat{g_1}(Y)-\hat{g_1}(X)}{C} \cdot (f(OPT) - f(X))\\
& \geq &  f(X) + \frac{a_{\min}}{C} \cdot (f(OPT) - f(X))\\
& = &  \left(1- \frac{a_{\min}}{C}\right) \cdot f(X) + \frac{a_{\min}}{C} \cdot f(OPT)\\
& \geq &  \left(1- \frac{a_{\min}}{C}\right) \cdot \left(1- e^{-\hat{g_1}^*/C} \right) \cdot f(OPT) + \frac{a_{\min}}{C} \cdot f(OPT)\\
& = & \left(1-  \left(1-\frac{a_{\min}}{C}\right) \cdot \left(e^{-\hat{g_1}^*/C} \right) \right) \cdot f(OPT)\\
& \geq & \left(1-  e^{-a_{\min}/C} \cdot e^{-\hat{g_1}^*/C}\right)\cdot f(OPT)\\
& = & \left(1-  e^{-(\hat{g_1}^* + a_{\min})/C}\right)\cdot f(OPT)
\end{eqnarray*}

holds.
Here the first inequality follows from $f$ being monotone and submodular and the second inequality uses that $g_1$ increases by at least $a_{\min}$.
The last inequality holds as $1-s \leq e^{-s}$ for all $s \in \mathds{R}$.
The $\hat{g_1}^*$-value for the considered solution, can increase at most $C/a_{\min}$ and therefore once included the search point $0^n$, the expected time to produce such solution
is $O(P_{\max}nC/a_{\min})$.

Let $x^*$ be the feasible solution of maximal cost included in the population after having increased the $\hat{g_1}^*$ at most $C/a_{\min}$ times as described above. Furthermore, let $v^*$ be the element with the largest $g_2$-value not included in $x^*$ and $\hat{x}$ the solution containing the single element with the largest $g_2$ produced from the search point $0^n$ in expected time $O(P_{\max}n)$. 

Let $r$ be the number of elements in a given solution. According to Lemma 1 and 2 in~\citet{DBLP:journals/corr/abs-1911-11451} , the maximal $\hat{g_1}^*$ deemed as feasible is at least
$$
C_1^*=C- \sqrt{\frac{(1 - \alpha) r \delta^2}{3\alpha}} 
$$
when using Chebyshev's inequality (Equation~\ref{thm:CHB}) and at least
$$
C_2^*=C - \sqrt{3\delta r \ln(1/\alpha)}
$$
when using the Chernoff bound (Equation~\ref{thm:CHF}). Note that and $\alpha$ and $\delta$ are constants. Furthermore, we have $r \cdot a_{\min} \leq C = \Omega(r)$ and $C/a_{\max} = \omega(1)$ for any feasible solution which implies $\sqrt{\frac{(1 - \alpha) r \delta^2}{3\alpha}} /C = o(1)$ and  $\sqrt{3\delta r \ln(1/\alpha)} / C = o(1)$.

For a fixed $C^*$-value, we have 

$$
 \left(1-  e^{-C^*/C} \right) \cdot f(OPT)
$$
which is at least 
$(1-e^{-{1 +o(1)}}) \cdot f(OPT) = (1-o(1) (1- 1/e) \cdot f(OPT)$ for the values of $C_1^*$ and $C_2^*$.

We have $\hat{g_1}(x^*)+ a(v^*)>C_1^*$ when working with Chebyshev's inequality and  $\hat{g_1}(x^*)+ a(v^*)>C_2^*$ when using Chernoff bound. In addition, $f(\hat{x}) \geq f(v^*)$ holds. 
Hence,
 $x^*$ or $\hat{x}$ is therefore a $(1/2-o(1))(1-1/e)$-approximation which completes the proof.
\end{proof}

For the special case of uniform IID weights, we have $a=a_{\max}=a_{\min}$ and $P_{max}\leq C/a +1$ and $\hat{g_1}(x) = a \cdot |x|_1$ which means that this value is determined by the number of chosen elements. 
Then the algorithm behaves as when using $g_1$ which we already  investigated in  Theorem~\ref{thm:iid}. This gives an upper bound on the expected runtime of $O(nk (k+\log n))$ to obtain a $(1-o(1))(1-1/e)$-approximation for the uniform IID case when working with the function $\hat{g_1}$ instead of $g_1$.

\section{Experimental Setup}
\label{sec5}
We carry out experimental investigations for different chance-constrained variants of monotone submodular optimization problems and describe the experimental setting in this section. For our investigations, we consider the influence maximization problem in social networks and the maximum coverage problem. Our primary goal is to compare \gsemo to \ga and \gga for optimizing monotone chance constrained submodular problems as proposed in \citet{DBLP:journals/corr/abs-1911-11451}. We will show that \gsemo using the multi-objective formulations introduced in this paper significantly outperforms \ga and \gga in terms of the quality of the results obtained. Furthermore, we also consider prominent evolutionary multi-objective algorithms such as SPEA2 ~\citep{zitzler2001spea2} and NSGA-II~\citep{DBLP:journals/tec/DebAPM02} using the multi-objective formulations introduced in this paper and show when they allow for additional improvements compared to \gsemo.

\subsection{The Influence Maximization Problem}
The influence maximization problem (IM) (see~\citep{DBLP:journals/toc/KempeKT15,DBLP:conf/kdd/LeskovecKGFVG07,DBLP:conf/ijcai/QianSYT17,DBLP:conf/aaai/ZhangV16} for detailed descriptions) is a key problem in social influence analysis and aims to find a set of most influential users in a large-scale social network. The task for IM is to select a set of users that maximize the spread of influence through a given social network, i.e., a graph of interactions and relationships within a group of users~\citep{chen2009efficient,DBLP:conf/kdd/KempeKT03}. The problem of influence maximization has been studied subject to deterministic constraints where selecting a set of users is associated with some costs and the selection of users has to be within a given budget~\citep{DBLP:conf/ijcai/QianSYT17}.

Formally, a social network is modeled as a directed graph $G=(V,E)$ where each node represents a user, and each edge $(u,v) \in E$ represents the possibility that users $u$ influence user $v$. Each edge $(u,v) \in E$ is assigned an edge probability $p_{u,v}$ which is the probability that user $u$ influences user $v$. The aim of the IM problem is to find a subset $X \subseteq V$ such that the expected number of activated nodes $E[I(X)]$ of $X$ is maximized. Given a cost function $c \colon V\rightarrow \R^+$ and a budget $C\ge 0$, the corresponding submodular optimization problem under chance constraints is given as
$$\argmax_{X\subseteq V} E[I(X)] \text{ s.t. } \Pr[c(X)> C]\leq \alpha.$$ 

Note that the cost of a selection $X$ is stochastic in this chance-constrained setting.
For influence maximization, we consider uniform cost constraints where each node has expected cost $1$. The expected cost of a solution is therefore $E_W(X)= |X|$.

In order to evaluate the algorithms on the chance-constrained influence maximization problem, we use the real-world data set with $400$ nodes and $1\,594$ edges provided in~\citet{DBLP:conf/ijcai/QianSYT17,digg2009}.

\subsection{The Maximum Coverage Problem}
\label{MCP}

The maximum coverage problem~\citep{DBLP:journals/ipl/KhullerMN99,DBLP:journals/jacm/Feige98} is an important NP-hard submodular optimization problem. We consider the chance-constrained version of the problem.
Given a set $U$ of elements, a collection $V = \{S_1,S_2,\ldots,S_n\}$ of subsets of $U$, a cost function $c \colon 2^V\rightarrow \R^+$
and a budget $C$, the goal is to find
$$\argmax_{X  \subseteq  V}  \{f(X) = |\cup_{S_i \in X} S_i| \text{ s.t. } \Pr(c(X) > C) \leq \alpha\}.$$

We consider linear cost functions. For the uniform case each set $S_i$ has expected cost $1$ and we have $E_W(X) = |\{i \mid S_i \in X\}|$. 
For the experiments with uniformly distributed weights having the same dispersion, each set $S_i$ has expected cost $|S_i|$ and we use $E_W(X)=\sum_{S_i \in X} |S_i|$. 

For our experiments, we investigate maximum coverage instances based on graphs. The $U$ elements consist of the vertices of the graph and
for each vertex, we generate a set which contains the vertex itself and its adjacent vertices. Hence, the expected cost of a node $v$ identifying a particular set is $a(v)=1 + deg(v)$ where $deg(v)$ is the degree of $v$ in the given graph.
For chance-constrained maximum coverage problem, we use the graphs frb30-15-01 ($450$ nodes, $17\,827$ edges) and frb35-17-01 ($595$ nodes, $27\,856$ edges) from~\citet{datasetsfrb}.

\subsection{Stochastic Settings}

We now describe experimental settings for the investigated evolutionary multi-objective algorithms.
In addition to \gsemo~\citep{Giel2003}, we also consider the application of state of the art evolutionary multi-objective algorithms, namely \nsga~\citep{DBLP:journals/tec/DebAPM02} for our chance-constrained problems. For each evolutionary algorithm run we allow $5\,000\,000$ fitness evaluations.
We run \nsga with parent population size $20$, offspring population size $10$, 
crossover probability $0.90$ and standard bit mutation for $500\,000$ generations.
Note that this implies that the parent population size of NSGA-II is smaller than the one of \gsemo if many trade-offs according to the given objective functions are possible. This is in particular the case when the constraint bound is large and many different cost values are possible for feasible solutions.
However, as the crowding distance focuses on the best value according to each objective, the best solution according to the given constrained single-objective problem is maintained and improved using the multi-objective formulation. 
To ensure a robust search behavior of NSGA-II and to guarantee a fair comparison among the algorithms, we performed experimental testing using various parameter combinations~\citep{DBLP:journals/tec/DebAPM02,wang2019parameterization,ferreira2023nsga}.

Our goal is to study different chance constraint settings in terms of the constraint bound $C$, the dispersion $\delta$, and the probability bound $\alpha$. We consider different benchmarks for chance-constrained versions of the maximum influence problems and the maximum coverage problem.
For each benchmark set, we study the performance in terms of the quality of solutions obtained by the  algorithms for different constraint bounds. We consider $C = 20, 50, 100$ for influence maximization, and $C = 10, 15, 20$ for maximum coverage problem, and  $C = 500$ for maximum coverage problem with degree based constraint.
For the experimental investigations of the algorithms and problems we consider all combinations of $\alpha = 0.1, 0.001$, and $\delta = 0.5, 1.0$ to gain a good understanding of the algorithm's performance for different chance-constrained settings.
In the case of the maximum coverage problem with degree based constraint we consider $\alpha = 0.1, 0.001$, and $\delta = 1.0, 10, 20, 30, 40, 50$.
It has been shown that Chebyshev's inequality leads to better results when $\alpha$ is relatively large and the Chernoff bounds gives better results for small $\alpha$~\citep{DBLP:conf/gecco/XieHAN019,DBLP:journals/corr/abs-1911-11451}. Therefore, we use Equation~\ref{thm:CHB} for $\alpha=0.1$ and Equation~\ref{thm:CHF} for $\alpha=0.001$ when computing the upper bound on the probability of a constraint violation. For each tested instance, we repeat the run $30$ times independently and report the minimum, maximum, average results and statistical tests.

In order to test statistical significance of the results we use the Kruskal-Wallis test  $95$\% confidence level in order to measure the statistical validity of our results. We apply the Bonferroni post-hoc statistical procedure that is used for multiple comparison of a control algorithm to two or more algorithms~\citep{Corder09}. $X^{(+)}$ is equivalent to the statement that the algorithm in the column outperformed algorithm $X$. $X^{(-)}$ is equivalent to the statement that X outperformed the algorithm given in the column. In the case when the algorithm $X$ does not appear, this means that no significant difference was determined between algorithms.

\begin{table}[!t]
\centering
\renewcommand{\arraystretch}{1.6} 
\renewcommand\tabcolsep{1.8pt} 
\caption{Comparison of results for influence maximization with uniform chance constraints in terms of the mean and statistical test of fitness values for \gsemo (2) and \nsga (3) on real-world data set (rows 1-12). Highest fitness values between \ga (1), \gsemo (2), and \nsga (3) are highlighted in {\textbf{bold face}}.
}
\label{tb:400Cheb}
\vspace{+4mm}
\begin{scriptsize} 

\begin{tabular}{@{}cccr|rrrrc|rrrrc}
\toprule  
     
\multirow{2}{*}{$C$} & \multirow{2}{*}{$\alpha$} &
\multirow{2}{*}{$\delta$} & \multicolumn{1}{c}{\bfseries GA (1)} & \multicolumn{5}{c}{\bfseries \gsemo (2)} & \multicolumn{5}{c}{\bfseries NSGA-II (3)}\\
\cmidrule(l{2pt}r{2pt}){5-9} \cmidrule(l{2pt}r{2pt}){10-14} 
 &  &  &  & \textbf{mean} & \textbf{min}
&\textbf{max} & \textbf{std} & \textbf{stat} &\textbf{mean} &\textbf{min} & \textbf{max}
 & \textbf{std}& \textbf{stat} \\
\midrule

\multirow{2}{*}{20}&0.1&0.5&51.51&\textbf{55.75}&54.44&56.85&0.5571&$1^{(+)}$&55.66&54.06&56.47 &0.5661&$1^{(+)}$\\    
&0.1&1.0&46.80&\textbf{50.65}&49.53&51.68&0.5704&$1^{(+)}$&50.54&49.61&52.01&0.6494&$1^{(+)}$\\

\multirow{2}{*}{50}&0.1&0.5&90.55&\textbf{94.54}&93.41&95.61&0.5390&$1^{(+)},3^{(+)}$&92.90&90.75&94.82&1.0445&$1^{(+)},2^{(-)}$\\    
&0.1&1.0&85.71&\textbf{88.63}&86.66&90.68&0.9010&$1^{(+)},3^{(+)}$&86.89&85.79&88.83&0.8479&$1^{(+)},2^{(-)}$\\

\multirow{2}{*}{100}&0.1&0.5&144.16&\textbf{147.28}&145.94&149.33&0.8830&$1^{(+)},3^{(+)}$&144.17&142.37&146.18&0.9902&$2^{(-)}$\\    
&0.1&1.0&135.61&\textbf{140.02}&138.65&142.52&0.7362&$1^{(+)},3^{(+)}$&136.58&134.80&138.21&0.9813&$2^{(-)}$\\
\midrule

\multirow{2}{*}{20}&0.001&0.5&48.19&\textbf{50.64}&49.10&51.74&0.6765&$1^{(+)}$&50.33&49.16&51.25&0.5762&$1^{(+)}$\\    
&0.001&1.0&39.50&\textbf{44.53}&43.63&45.55&0.4687&$1^{(+)}$&44.06&42.18&45.39&0.7846&$1^{(+)}$

\\

\multirow{2}{*}{50}&0.001&0.5&75.71&\textbf{80.65}&78.92&82.19&0.7731&$1^{(+)}$&80.58&79.29&81.63 &0.6167&$1^{(+)}$\\    
&0.001&1.0&64.49&69.79&68.89&71.74&0.6063&$1^{(+)}$&\textbf{69.96}&68.90&71.05 &0.6192&$1^{(+)}$\\

\multirow{2}{*}{100}&0.001&0.5&116.05&\textbf{130.19}&128.59&131.51&0.7389&$1^{(+)},3^{(+)}$&127.50&125.38&129.74&0.9257&$1^{(+)},2^{(-)}$\\    
&0.001&1.0&96.18&\textbf{108.95}&107.26&109.93&0.6466&$1^{(+)},3^{(+)}$&107.91&106.67&110.17 &0.7928&$1^{(+)},2^{(-)}$\\
 \midrule

\end{tabular}

\end{scriptsize}
\end{table}    
\begin{table}[!t]
\centering
\renewcommand{\arraystretch}{1.6}  
\renewcommand\tabcolsep{1.8pt} 
\caption{Comparison of results for maximum coverage with uniform chance constraints in terms of the mean and statistical test of fitness values for \gsemo (2), and \nsga (3) for the graphs frb30-15-01 (rows 1-12) and frb35-17-01 dataset (rows 13-24). Highest fitness values between \ga (1), \gsemo (2), and \nsga (3) are highlighted in {\textbf{bold face}}.
} 
\label{tb:MCP1}
\vspace{+4mm}
\begin{scriptsize}

\begin{tabular}{@{}cccc|rrrrc|rrrrc}

\toprule
     
\multirow{2}{*}{$C$} & \multirow{2}{*}{$\alpha$} &
\multirow{2}{*}{$\delta$} & \multicolumn{1}{c}{\bfseries GA (1)} & \multicolumn{5}{c}{\bfseries \gsemo (2)} & \multicolumn{5}{c}{\bfseries NSGA-II (3)}\\
\cmidrule(l{2pt}r{2pt}){5-9} \cmidrule(l{2pt}r{2pt}){10-14} 
 &  &  &  & \textbf{mean} & \textbf{min}
&\textbf{max} & \textbf{std}& \textbf{stat} & \textbf{mean} &\textbf{min} & \textbf{max}
 & \textbf{std} & \textbf{stat} \\
\midrule 

\multirow{2}{*}{10}&0.1&0.5&371.00&\textbf{377.23}&371.00&379.00&1.8323&$1^{(+)}$&376.00&371.00&379.00&2.5596&$1^{(+)}$\\    
&0.1&1.0&321.00&\textbf{321.80}&321.00&325.00&1.5625&$1^{(+)}$& 321.47&321.00& 325.00&1.2521\\

\multirow{2}{*}{15}&0.1&0.5&431.00&\textbf{439.60}&435.00&442.00&1.7340&$1^{(+)},3^{(+)}$&437.57&434.00&441.00&1.7555&$1^{(+)},2^{(-)}$ \\    
&0.1&1.0&403.00&\textbf{411.57}&408.00&414.00&1.7750&$1^{(+)}$&410.67 &404.00&414.00&2.5098&$1^{(+)}$ \\

\multirow{2}{*}{20}&0.1&0.5&446.00&\textbf{450.07}&448.00&451.00&0.8277&$1^{(+)},3^{(+)}$&448.27&445.00&451.00&1.3113&$1^{(+)},2^{(-)}$ \\    
&0.1&1.0&437.00&\textbf{443.87}&441.00&446.00&1.2794&$1^{(+)},3^{(+)}$&441.37&438.00 &444.00 &1.6914&$1^{(+)},2^{(-)}$ \\
\midrule 

\multirow{2}{*}{10}&0.001&0.5&348.00&\textbf{352.17}&348.00&355.00&2.4081&$1^{(+)}$&350.80&348.00&355.00&2.8935&$1^{(+)}$\\    
&0.001&1.0&321.00&\textbf{321.67}&321.00&325.00&1.5162&$1^{(+)}$&321.33 &321.00&325.00&1.0613\\

\multirow{2}{*}{15}&0.001&0.5&414.00&\textbf{423.90}&416.00&426.00&2.4824&$1^{(+)}$&422.67& 419.00&426.00&2.2489&$1^{(+)}$\\    
&0.001&1.0&371.00&\textbf{376.77}&371.00&379.00&1.8134&$1^{(+)}$&376.33& 371.00 & 379.00&2.6824&$1^{(+)}$\\

\multirow{2}{*}{20}&0.001&0.5&437.00&\textbf{443.53}&440.00&445.00&1.1958&$1^{(+)},3^{(+)}$&440.23&437.00&443.00&1.6955&$1^{(+)},2^{(-)}$\\    
&0.001&1.0&414.00&\textbf{424.00}&420.00&426.00&1.7221&$1^{(+)}$&422.50&417.00&426.00&2.5291&$1^{(+)}$\\
\midrule

\midrule

\multirow{2}{*}{10}&0.1&0.5&448.00&\textbf{458.80}&451.00&461.00&3.3156&$1^{(+)}$&457.97&449.00&461.00&4.1480&$1^{(+)}$\\    
&0.1&1.0&376.00&\textbf{383.33}&379.00&384.00&1.7555&$1^{(+)}$&382.90&379.00&384.00&2.0060&$1^{(+)}$ \\

\multirow{2}{*}{15}&0.1&0.5&559.00&\textbf{559.33}&555.00&562.00&2.0057&$3^{(+)}$&557.23&551.00 &561.00&2.4309&$1^{(-)},2^{(-)}$\\    
&0.1&1.0&503.00&\textbf{507.80}&503.00&509.00&1.1567&$1^{(+)}$&507.23&502.00&509.00 &1.8323&$1^{(+)}$\\

\multirow{2}{*}{20}&0.1&0.5&587.00&\textbf{587.20}&585.00&589.00&1.2149&$3^{(+)}$&583.90&580.00&588.00&1.9360&$1^{(-)},2^{(-)}$\\    
&0.1&1.0&569.00&\textbf{569.13}&566.00&572.00&1.4559&$3^{(+)}$&565.30&560.00&569.00&2.1520&$1^{(-)},2^{(-)}$\\
\midrule 

\multirow{2}{*}{10}&0.001&0.5&413.00&\textbf{423.67}&418.00&425.00&1.8815&$1^{(+)}$&422.27&416.00&425.00&2.6121&$1^{(+)}$\\    
&0.001&1.0&376.00&\textbf{383.70}&\textbf379.00&384.00&1.1492&$1^{(+)}$&381.73&377.00&384.00&2.6514 &$1^{(+)}$\\

\multirow{2}{*}{15}&0.001&0.5&526.00&\textbf{527.97}&525.00&532.00&2.1573&$1^{(+)}$&527.30&520.00&532.00&2.7436\\    
&0.001&1.0&448.00&\textbf{458.87}&453.00&461.00&2.9564&$1^{(+)}$&457.10&449.00  &461.00 &4.1469&$1^{(+)}$\\

\multirow{2}{*}{20}&0.001&0.5&568.00&\textbf{568.87}&565.00&572.00&1.5025&$3^{(+)}$&564.60&560.00&570.00&2.7618&$1^{(-)}$,$2^{(-)}$\\    
&0.001&1.0&526.00&\textbf{528.03}&525.00&530.00&1.8843&$1^{(+)}$&527.07&522.00&530.00&2.2427\\
\midrule
 \end{tabular}
\end{scriptsize}
\end{table}
\vspace{-0.4cm}
\begin{table*}[t!]
\renewcommand{\arraystretch}{1.4} 
\renewcommand\tabcolsep{6.4pt}
\centering

\caption{
Results in terms of the mean value for maximum coverage with degree-based chance constraints for \gga (1), \gsemo $g$ (2) and $\hat{g}$ (3) on graphs frb30-15-01 (rows 1-12) and frb35-17-01 dataset (rows 13-24). Highest mean fitness values are highlighted in {\textbf{bold face}}.
Statistical tests also include comparison to \nsga using $g$ (4) and $\hat{g}$ (5) and SPEA2 using $g$ (6) and $\hat{g}$ (7)  (see Table~\ref{tb:MCPoutdegree_SPEA} for mean values of results.)
}
\label{tb:MCPoutdegree}
\vspace{+4mm}
\begin{scriptsize}
\begin{tabular}{@{}ccrc|cccc}

    \toprule
     
\multirow{2}{*}{$C$} & \multirow{2}{*}{$\alpha$} &
\multirow{2}{*}{$\delta$} & \multicolumn{1}{c}{\bfseries \gga (1)} & \multicolumn{4}{c}{\bfseries \gsemo} \\
\cmidrule(l{2pt}r{2pt}){5-8} 
\cmidrule(l{2pt}r{2pt}){4-4}
  &  & &  & \textbf{$g~(2)$} & \textbf{$stat$}  & \textbf{$\hat{g}~(3)$ } & \textbf{$stat$} 
 \\
\midrule

\multirow{6}{*}{500}&0.1&1&360&\textbf{363.77}&$1^{(+)},4^{(-)},5^{(-)}, 6^{(-)},7^{(-)}$&363.20&$1^{(+)},4^{(-)},5^{(-)},6^{(-)},7^{(-)}$ \\

&0.1&10&346&\textbf{352.90}&$1^{(+)}, 4^{(-)},5^{(-)},6^{(-)},7^{(-)}$&351.07&$1^{(+)}, 4^{(-)},5^{(-)},6^{(-)},7^{(-)}$ \\

&0.1&20&331&\textbf{338.33}&$1^{(+)},4^{(-)},5^{(-)},6^{(-)},7^{(-)}$&336.23&$1^{(+)},4^{(-)},5^{(-)},6^{(-)},7^{(-)}$ \\

&0.1&30&316&\textbf{322.00}&$1^{(+)},4^{(-)},6^{(-)},7^{(-)}$&321.83&$1^{(+)},4^{(-)},6^{(-)},7^{(-)}$ \\

&0.1&40&289&\textbf{307.37}&$1^{(+)},6^{(-)}$&306.53&$1^{(+)},6^{(-)}$ \\

&0.1&42&282&\textbf{302.37}&$1^{(+)},4^{(-)},5^{(-)}, 6^{(-)},7^{(-)}$&302.13&$1^{(+)},4^{(-)},5^{(-)},6^{(-)},7^{(-)}$ \\ 
\midrule

\multirow{6}{*}{500}&0.001&1&362&\textbf{363.17}&$1^{(+)},4^{(-)},5^{(-)},6^{(-)},7^{(-)}$&362.93&$1^{(+)},4^{(-)},5^{(-)},6^{(-)},7^{(-)}$ \\

&0.001&10&337&\textbf{348.30}&$1^{(+)},4^{(-)},5^{(-)}$&347.73&$1^{(+)},4^{(-)},5^{(-)}$ \\

&0.001&20&316&\textbf{327.03}&$1^{(+)},4^{(-)},5^{(-)},6^{(-)},7^{(-)}$&327.00&$1^{(+)},4^{(-)},5^{(-)},6^{(-)},7^{(-)}$ \\

&0.001&30&282&315.63&$1^{(+)}$&\textbf{315.77}&$1^{(+)}$ \\

&0.001&40&259&\textbf{295.20}&$1^{(+)},5^{(-)},6^{(-)},7^{(-)}$&294.93&$1^{(+)},4^{(-)},5^{(-)},6^{(-)},7^{(-)}$ \\

&0.001&42&259&\textbf{289.33}&$1^{(+)},4^{(-)},5^{(-)},6^{(-)},7^{(-)}$&289.13&$1^{(+)},4^{(-)},5^{(-)},6^{(-)},7^{(-)}$\\

\midrule
\midrule

\multirow{6}{*}{500}&0.1&1&402&\textbf{406.40}&$1^{(+)},4^{(-)},5^{(-)},6^{(-)},7^{(-)}$&406.67&$1^{(+)},4^{(-)},5^{(-)},6^{(-)},7^{(-)}$ \\

&0.1&10&375&\textbf{386.87}&$1^{(+)},4^{(-)},5^{(-)},6^{(-)},7^{(-)}$&385.07&$1^{(+)},4^{(-)},5^{(-)},6^{(-)},7^{(-)}$ \\

&0.1&20&354&\textbf{370.13}&$1^{(+)},4^{(-)},5^{(-)},6^{(-)},7^{(-)}$&369.80&$1^{(+)},4^{(-)},5^{(-)},6^{(-)},7^{(-)}$ \\

&0.1&30&335&\textbf{351.70}&$1^{(+)},4^{(-)},6^{(-)},7^{(-)}$&349.23&$1^{(+)},4^{(-)},5^{(-)},6^{(-)},7^{(-)}$\\

&0.1&40&302&\textbf{328.20}&$1^{(+)},4^{(-)},5^{(-)},6^{(-)},7^{(-)}$&327.57&$1^{(+)},4^{(-)},5^{(-)},6^{(-)},7^{(-)}$\\

&0.1&50&252&\textbf{316.97}&$1^{(+)},6^{(-)}$&316.60& $1^{(+)},6^{(-)}$\\

\midrule

\multirow{6}{*}{500}&0.001&1&397&\textbf{405.70}&$1^{(+)},4^{(-)},5^{(-)},6^{(-)},7^{(-)}$&405.40&$1^{(+)},4^{(-)},5^{(-)},6^{(-)},7^{(-)}$\\

&0.001&10&354&\textbf{384.50}&$1^{(+)},4^{(-)},5^{(-)},6^{(-)},7^{(-)}$&384.37&$1^{(+)},4^{(-)},5^{(-)},6^{(-)},7^{(-)}$\\

&0.001&20&347&\textbf{364.37}&$1^{(+)},6^{(-)},7^{(-)}$&363.83&$1^{(+)},6^{(-)},7^{(-)}$\\

&0.001&30&302&\textbf{338.70}&$1^{(+)},4^{(-)},5^{(-)},6^{(-)},7^{(-)}$&338.00&$1^{(+)},4^{(-)},5^{(-)},6^{(-)},7^{(-)}$\\

&0.001&40&252&\textbf{323.83}&$1^{(+)}$&323.73&$1^{(+)}$\\

&0.001&50&252&\textbf{317.00}&$1^{(+)},6^{(-)},7^{(-)}$&316.83&$1^{(+)},6^{(-)},7^{(-)}$\\

\midrule 

\end{tabular}
\end{scriptsize}
\end{table*}
\begin{sidewaystable}
\centering
\tiny

\renewcommand{\arraystretch}{1.6} 
\renewcommand\tabcolsep{5.8pt} 
\centering

\caption{
Results in terms of the mean value for maximum coverage with degree-based chance constraints for \nsga using $g$ (4), $\hat{g}$ (5) and SPEA2 using $g$ (6) and $\hat{g}$ (7) on graphs frb30-15-01 (rows 1-12) and frb35-17-01 dataset (rows 13-24). Highest fitness values in comparison to results from Table~\ref{tb:MCPoutdegree} are highlighted in \colorbox{gray!20} {color} and highest fitness values between \nsga and \spea are highlighted in {\textbf{bold face}}. Statistical tests also include comparison to \gga (1) and \gsemo using $g$ (2) and $\hat{g}$ (3) (see Table~\ref{tb:MCPoutdegree} for mean values of results.)
}

\label{tb:MCPoutdegree_SPEA}
\begin{tiny}
\begin{tabular}{@{}ccc|cccc|cccc}

    \toprule
     
\multirow{2}{*}{$C$} & \multirow{2}{*}{$\alpha$} &
\multirow{2}{*}{$\delta$}  & \multicolumn{4}{c}{\bfseries NSGA-II} & \multicolumn{4}{c}{\bfseries SPEA2}\\
\cmidrule(l{2pt}r{2pt}){4-7} 
\cmidrule(l{2pt}r{2pt}){8-11} 
 &  &   
  & \textbf{$g~(4)$} & \textbf{$stat$}  &\textbf{$\hat{g}~(5)$ } & \textbf{$stat$} & \textbf{$g~(6)$} & \textbf{$stat$}  &\textbf{$\hat{g}~(7)$ } & \textbf{$stat$}
 \\
\midrule

\multirow{6}{*}{500}&0.1&1&375.90&$1^{(+)},2^{(+)},3^{(+)},6^{(-)}, 7^{(-)}$&375.70&$1^{(+)},2^{(+)},3^{(+)},6^{(-)}, 7^{(-)}$&380.57&$1^{(+)},2^{(+)},3^{(+)},4^{(+)}, 5^{(+)}$&\cellcolor{gray!20}\textbf{380.63}&$1^{(+)},2^{(+)},3^{(+)},4^{(+)}, 5^{(+)}$\\    

&0.1&10&362.00&$1^{(+)},2^{(+)},3^{(+)},6^{(-)}, 7^{(-)}$ &360.47&$1^{(+)},2^{(+)},3^{(+)},6^{(-)}, 7^{(-)}$&365.03&$1^{(+)},2^{(+)},3^{(+)},4^{(+)}, 5^{(+)}$&\cellcolor{gray!20}\textbf{366.50}&$1^{(+)},2^{(+)},3^{(+)},4^{(+)}, 5^{(+)}$\\ 

&0.1&20&342.47&$1^{(+)},2^{(+)},3^{(+)},6^{(-)}, 7^{(-)}$ &342.37&$1^{(+)},2^{(+)},3^{(+)},6^{(-)}, 7^{(-)}$&346.57&$1^{(+)},2^{(+)},3^{(+)},4^{(+)}, 5^{(+)}$&\cellcolor{gray!20}\textbf{346.77}&$1^{(+)},2^{(+)},3^{(+)},4^{(+)}, 5^{(+)}$\\ 

&0.1&30&324.60&$1^{(+)},2^{(+)},3^{(+)},6^{(-)}$ &323.57&$1^{(+)},6^{(-)},7^{(-)}$&\cellcolor{gray!20}\textbf{327.37}&$1^{(+)},2^{(+)},3^{(+)},4^{(+)}, 5^{(+)}$&326.93&$1^{(+)},2^{(+)},3^{(+)}, 5^{(+)}$ \\ 

&0.1&40&308.60&$1^{(+)}$ &306.37&$1^{(+)},6^{(-)}$&\cellcolor{gray!20}\textbf{310.47}&$1^{(+)},2^{(+)},3^{(+)},5^{(+)}$&309.14&$1^{(+)},2^{(+)},3^{(+)}$\\ 

&0.1&42&305.43&$1^{(+)},2^{(+)},3^{(+)},6^{(-)}, 7^{(-)}$ &305.00&$1^{(+)},2^{(+)},3^{(+)},6^{(-)}, 7^{(-)}$&\cellcolor{gray!20}\textbf{307.73}&$1^{(+)},2^{(+)},3^{(+)},4^{(+)}, 5^{(+)}$&307.14&$1^{(+)},2^{(+)},3^{(+)},4^{(+)}, 5^{(+)}$ \\ 
\midrule

\multirow{6}{*}{500}&0.001&1&376.13&$1^{(+)},2^{(+)},3^{(+)},6^{(-)},7^{(-)}$ &374.77&$1^{(+)},2^{(+)},3^{(+)},6^{(-)}, 7^{(-)}$&\cellcolor{gray!20}\textbf{380.57}&$1^{(+)},2^{(+)},3^{(+)},4^{(+)}, 5^{(+)}$&379.80&$1^{(+)},2^{(+)},3^{(+)},4^{(+)}, 5^{(+)}$ \\ 

&0.001&10&354.16&$1^{(+)},2^{(+)},3^{(+)},6^{(-)}, 7^{(-)}$ &352.80&$1^{(+)},2^{(+)},3^{(+)},6^{(-)},7^{(-)}$&357.00&$1^{(+)},2^{(+)},3^{(+)},4^{(+)}, 5^{(+)}$&\cellcolor{gray!20}\textbf{357.33}&$1^{(+)},2^{(+)},3^{(+)},4^{(+)}, 5^{(+)}$\\ 

&0.001&20&331.83&$1^{(+)},2^{(+)},3^{(+)},6^{(-)}, 7^{(-)}$ &331.20&$1^{(+)},2^{(+)},3^{(+)},6^{(-)}, 7^{(-)}$&334.33&$1^{(+)},2^{(+)},3^{(+)},4^{(+)}, 5^{(+)}$&\cellcolor{gray!20}\textbf{334.90}&$1^{(+)},2^{(+)},3^{(+)},4^{(+)}, 5^{(+)}$ \\ 

&0.001&30&313.57&$1^{(+)}$ &312.47&$1^{(+)}$&315.50&$1^{(+)}$&\cellcolor{gray!20}\textbf{315.83}&$1^{(+)}$ \\

&0.001&40&297.73&$1^{(+)},2^{(+)},6^{(-)}, 7^{(-)}$ &297.70&$1^{(+)},2^{(+)},3^{(+)},6^{(-)}, 7^{(-)}$&\cellcolor{gray!20}\textbf{300.97}&$1^{(+)},2^{(+)},3^{(+)},4^{(+)}, 5^{(+)}$&300.63&$1^{(+)},2^{(+)},3^{(+)},4^{(+)}, 5^{(+)}$ \\ 

&0.001&42&292.97&$1^{(+)},2^{(+)},3^{(+)},6^{(-)}, 7^{(-)}$&292.93&$1^{(+)},2^{(+)},3^{(+)},6^{(-)}, 7^{(-)}$&295.73&$1^{(+)},2^{(+)},3^{(+)},4^{(+)},5^{(+)}$&\cellcolor{gray!20}\textbf{296.40}&$1^{(+)},2^{(+)},3^{(+)},4^{(+)}, 5^{(+)}$\\ 

\midrule
\midrule

\multirow{6}{*}{500}&0.1&1&421.57&$1^{(+)},2^{(+)},3^{(+)},6^{(-)}, 7^{(-)}$ &420.53&$1^{(+)},2^{(+)},3^{(+)},6^{(-)}, 7^{(-)}$&\cellcolor{gray!20}\textbf{427.83}&$1^{(+)},2^{(+)},3^{(+)},4^{(+)}, 5^{(+)}$&427.50&$1^{(+)},2^{(+)},3^{(+)},4^{(+)}, 5^{(+)}$ \\

&0.1&10&401.77&$1^{(+)},2^{(+)},3^{(+)},6^{(-)}, 7^{(-)}$ &399.70&$1^{(+)},2^{(+)},3^{(+)},6^{(-)}, 7^{(-)}$&407.57&$1^{(+)},2^{(+)},3^{(+)},4^{(+)}, 5^{(+)}$&\cellcolor{gray!20}\textbf{409.13}&$1^{(+)},2^{(+)},3^{(+)},4^{(+)}, 5^{(+)}$ \\

&0.1&20&376.47&$1^{(+)},2^{(+)},3^{(+)},6^{(-)}, 7^{(-)}$ &375.93&$1^{(+)},2^{(+)},3^{(+)},6^{(-)}, 7^{(-)}$&380.37&$1^{(+)},2^{(+)},3^{(+)},4^{(+)}, 5^{(+)}$&\cellcolor{gray!20}\textbf{381.03}&$1^{(+)},2^{(+)},3^{(+)},4^{(+)}, 5^{(+)}$ \\

&0.1&30&355.47&$1^{(+)},2^{(+)},3^{(+)},6^{(-)}, 7^{(-)}$  &354.10&$1^{(+)},3^{(+)},6^{(-)}, 7^{(-)}$&\cellcolor{gray!20}\textbf{358.87}&$1^{(+)},2^{(+)},3^{(+)},4^{(+)}, 5^{(+)}$&358.33&$1^{(+)},2^{(+)},3^{(+)},4^{(+)}, 5^{(+)}$\\

&0.1&40&333.73&$1^{(+)},2^{(+)},3^{(+)}$ &332.37&$1^{(+)},2^{(+)},3^{(+)},6^{(-)}, 7^{(-)}$&335.47&$1^{(+)},2^{(+)},3^{(+)}, 5^{(+)}$&\cellcolor{gray!20}\textbf{335.53}&$1^{(+)},2^{(+)},3^{(+)},5^{(+)}$ \\

&0.1&50&317.47&$1^{(+)},6^{(-)}$ &317.13&$1^{(+)},6^{(-)}$&\cellcolor{gray!20}\textbf{321.27}&$1^{(+)},2^{(+)},3^{(+)},4^{(+)}, 5^{(+)},7^{(+)}$&318.03&$1^{(+)},2^{(+)},3^{(+)},6^{(-)}$\\

\midrule 

\multirow{6}{*}{500}&0.001&1&425.86&$1^{(+)},2^{(+)},3^{(+)},6^{(-)}, 7^{(-)}$ &425.10&$1^{(+)},2^{(+)},3^{(+)},6^{(-)}, 7^{(-)}$&427.37&$1^{(+)},2^{(+)},3^{(+)},4^{(+)}, 5^{(+)}$&\cellcolor{gray!20}\textbf{427.47}&$1^{(+)},2^{(+)},3^{(+)},4^{(+)}, 5^{(+)}$ \\

&0.001&10&394.20&$1^{(+)},2^{(+)},3^{(+)},6^{(-)}, 7^{(-)}$ &392.43&$1^{(+)},2^{(+)},3^{(+)},6^{(-)}, 7^{(-)}$&398.27&$1^{(+)},2^{(+)},3^{(+)},4^{(+)}, 5^{(+)}$&\cellcolor{gray!20}\textbf{398.70}&$1^{(+)},2^{(+)},3^{(+)},4^{(+)}, 5^{(+)}$ \\ 

&0.001&20&367.03&$1^{(+)},6^{(-)}, 7^{(-)}$ &365.77&$1^{(+)},6^{(-)}, 7^{(-)}$&368.90&$1^{(+)},2^{(+)},3^{(+)},4^{(+)}, 5^{(+)}$&\cellcolor{gray!20}\textbf{369.03}&$1^{(+)},2^{(+)},3^{(+)},4^{(+)}, 5^{(+)}$ \\

&0.001&30&345.30&$1^{(+)},2^{(+)},3^{(+)}$ &345.10&$1^{(+)},2^{(+)},3^{(+)}$&\cellcolor{gray!20}\textbf{347.47}&$1^{(+)},2^{(+)},3^{(+)}$&346.87&$1^{(+)},2^{(+)},3^{(+)}$ \\

&0.001&40&324.07&$1^{(+)}$ &323.73&$1^{(+)}$&325.73&$1^{(+)}$&\cellcolor{gray!20}\textbf{325.90}&$1^{(+)}$ \\

&0.001&50&317.87&$1^{(+)},6^{(-)}, 7^{(-)}$ &317.43&$1^{(+)},6^{(-)}, 7^{(-)}$&\cellcolor{gray!20}\textbf{322.00}&$1^{(+)},2^{(+)},3^{(+)},4^{(+)}, 5^{(+)}$&321.43&$1^{(+)},2^{(+)},3^{(+)},4^{(+)}, 5^{(+)}$\\
\midrule

\end{tabular}
\end{tiny}
\end{sidewaystable}

\section{Experimental Results for Uniform IID Weights} 
\label{sec6}

The first type of cost constraints we consider is given by uniform independent, identically distributed (IID) weights. For our experiments, all items have IID weights $W(s) \in [1-\delta,1+\delta]$ ($\delta \le 1$). 

We consider the results for \ga, \gsemo and \nsga based on Chebyshev's inequality and Chernoff bounds. Table~\ref{tb:400Cheb} shows the results for the influence maximization problem obtained by \ga, \gsemo and \nsga for all combinations of $\alpha$ and $\delta$. The results show that \gsemo obtains the highest mean values compared to the results obtained by \ga and \nsga. Furthermore, the statistical tests show that for most of the combinations of $\alpha$ and $\delta$, \gsemo and \nsga significantly outperform \ga. The solutions obtained by \gsemo have significantly better performance than \nsga in the case of a high budget, i.e., for $C$ = $100$. A possible explanation for this is that the relatively small population size of $20$ for \nsga does not allow to construct solutions in a greedy fashion as it is possible for \ga and \gsemo. Note that the crowding distance within NSGA-II creates larger gaps in the objective space when the population size is not sufficient to cover the whole Pareto front.

Next, we perform experimental investigations for the maximum coverage problem
with uniform chance constraints for the graphs frb30-15-01 and frb35-17-01. The experimental results are shown in Table~\ref{tb:MCP1}. We compare the results for the maximum coverage problem of \ga, \gsemo and \nsga based on Chebyshev's inequality and Chernoff bound. Similar to the previous investigations, we observe that \gsemo based on Chebyshev and Chernoff chance estimates outperforms \ga and \nsga for all $\delta$ values. Furthermore, \gsemo statistically outperforms \ga for most of the settings. For the other settings there is no statistical significant difference in terms of the results for \gsemo and \ga, i.e., for the graph frb35-17-01 for budget $C$ = $20$, $\alpha$ = $0.1$ and $\delta$ = $0.5, 1.0$. NSGA-II is outperforming \ga for most of the examined settings and the majority of the results are statistically significant. However, \nsga performs worse than \ga for frb35-17-01 when $C=20$, $\alpha=0.1$ and $\delta$ = $0.5, 1.0$.

\section{Experimental Results for Uniformly Distributed Weights with Same Dispersion}
\label{sec7}

We now consider the algorithms for the setting where the expected weight $a(v)$ of a node $v$ is given as $a(v) = 1 +deg(v)$ where $deg(v)$ is the degree of node $v$ in the given graph $G$ (see also Section~\ref{MCP}). 

Table~\ref{tb:MCPoutdegree} shows the results obtained for the maximum coverage problem by \gga and \gsemo based on Chernoff bounds and Chebyshev's inequality for the combinations of $\alpha$ and $\delta$.
The results show that \gsemo outperforms \gga. Overall, the use of the fitness function $g$ when comparing the use of the fitness function $\hat{g}$ for \gsemo achieves better performance. 
This is interesting as it shows that the algorithms are not prevented from making progress when using the tail inequalities for the investigated experimental settings. The slightly better results for $g$ compared to $\hat{g}$ also justify the use of tail inequalities as part of the constraint evaluation.
Furthermore, the statistical tests show that for most combinations of $\alpha$ and $\delta$ \gsemo significantly outperforms \gga.
For example, for the graph frb35-17-01, $\delta = 40$ and $\alpha = 0.1, 0.001$, the fitness values of \gsemo using $g$ are $328.20$ and $323.83$ compared to 302 and 252 for \gga.

The results for \nsga using $g$ and $\hat{g}$ for all combinations of $\alpha$ and $\delta$ values are shown in Table~\ref{tb:MCPoutdegree_SPEA}
Overall, \nsga outperforms \gsemo in almost all settings and nearly all results are statistically significant. The only instances where \nsga 
does not perform better based on statistical tests are when considering results obtained for $\delta=40$, $\alpha =0.1$ and $\delta=30$, $\alpha=0.001$ for the graph frb30-15-01. Looking at the results for the graph frb35-17-01, we can see that there are no significantly differences between these both algorithms when considering settings such as $\delta$ = $20, 40, 50$ and $\alpha$ = $0.001$.
We also investigate the
impact of the use larger population sizes on the performance of \nsga for maximum coverage with degree-based chance constraints on the graphs frb30-15-01 and frb35-17-01 in order to see whether this leads to further improvements. We consider population sizes $\mu =20, 50, 100, 200$, offspring population size $\frac{\mu}{2}$ for $\alpha = 0.001$ and $\delta = 1, 10, 20, 30, 40, 42, 50$ for the graphs. The number of fitness evaluations is $5,000,000$ as in previous settings. 
The results are shown in Figure~\ref{fig:nsga-frb30_size} as the mean fitness values among $30$ independent runs. Our results indicate that as the population size increases, the fitness values decrease for all $\delta$ values on both graphs. This means that larger populations do not improve the performance of \nsga for the investigated instances. 
\begin{figure}[!t]
\centering
\includegraphics[width=0.49\textwidth]{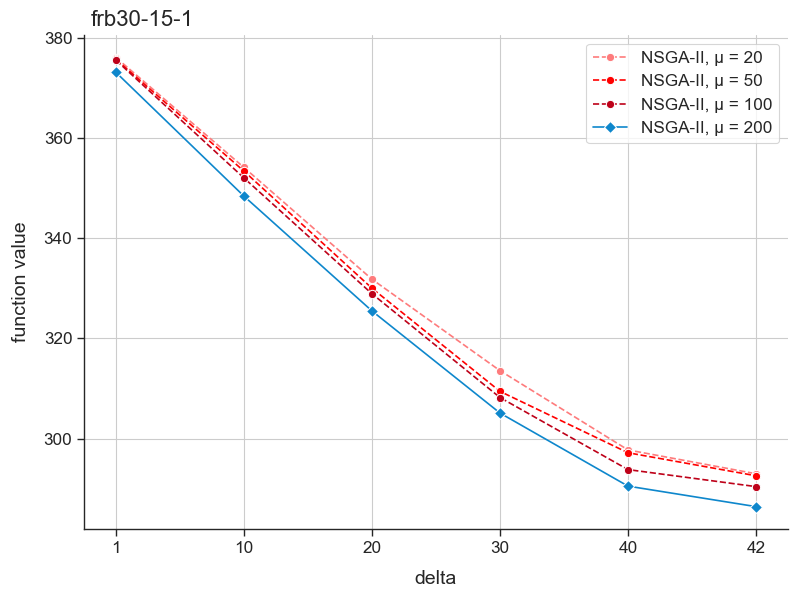}
\includegraphics[width=0.49\textwidth]{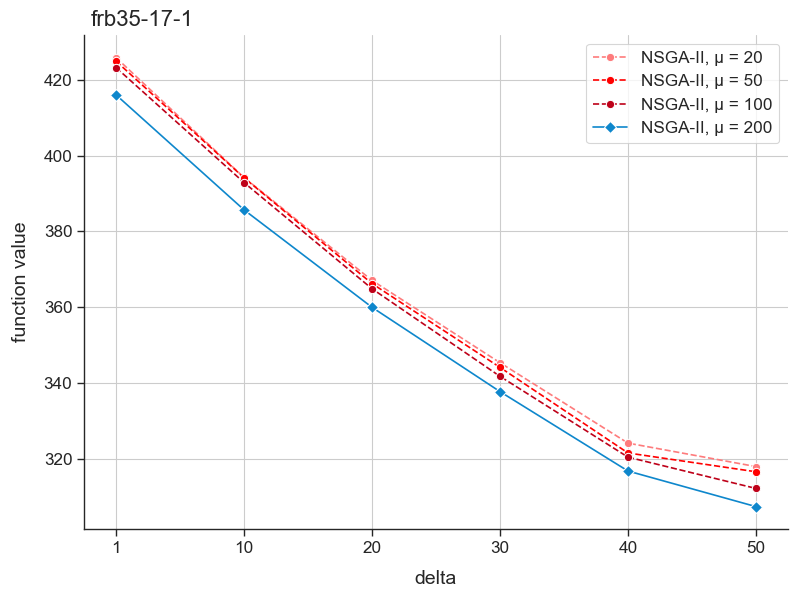}
\caption{Mean values for maximum coverage with degree-based chance constraints for \nsga on graph frb30-15-01 and frb35-17-01 (from left) for parent population size $\mu$ = $20, 50, 100, 200$ and offspring population size $\mu/2$. 
}
\label{fig:nsga-frb30_size}
\end{figure}
Furthermore, we investigate the performance of \spea as an alternative to, and in comparison with \nsga. We run \spea with parent population size $20$, offspring population size $10$, 
uniform crossover probability $0.90$ and standard bit mutation for $500\,000$ generations. 
Table~\ref{tb:MCPoutdegree_SPEA} displays the results obtained by \spea for $g$ and $\hat{g}$ for all combinations of $\alpha$ and $\delta$ values in additional to the ones obtained by \nsga. Overall, when comparing all results that are shown in Table~\ref{tb:MCPoutdegree} and Table~\ref{tb:MCPoutdegree_SPEA}, we observe that \spea obtains the highest mean fitness values in comparison to the results obtained by \gga, \gsemo and \nsga. 
In particular, the solutions obtained by \spea have significantly superior performance compared to \gsemo and \ga, especially when considering the case of a low $\delta$ value, i.e., for $\delta$ = $1, 10$ on both benchmarks. 
Furthermore, we do not observe significant differences when using the fitness functions $g$ and $\hat{g}$ for \spea.
These results indicate the effectiveness of \spea in solving the submodular chance-constrained maximum coverage problem. When comparing the use of the fitness functions $g$ and $\hat{g}$ in \nsga with the use of the fitness functions $g$ and $\hat{g}$ in \spea, we observe that \spea significantly outperforms \nsga in the most of the considered benchmark instances.
Overall, the use of fitness functions $g$ and $\hat{g}$ for \spea achieves the best results across various instances in majority of cases, demonstrating its superior performance.
%

\section{Conclusions}
Chance constraints involve stochastic components and require that a constraint is violated only with a small probability. We carried out a first runtime analysis on the optimization of monotone submodular functions with chance constraints. Our results show that \gsemo using a multi-objective formulation of the problem based on tail inequalities is able to achieve the same approximation guarantee as recently studied greedy approaches. Furthermore, our experimental results for the uniform IID case show that \gsemo computes significantly better solutions than the greedy approach and it is often outperforming NSGA-II. 
In the case of uniformly distributed weights with the same dispersion, our experimental results for degree-based constraints show that the generalized greedy algorithm is outperformed by all evolutionary multi-objective algorithms, and that \spea achieves the highest performance on the majority of the considered benchmark instances. 
For future work, it would be interesting to analyze other probability distributions for chance-constrained monotone submodular functions. A natural next step would be to examine uniformly distributed weights with different dispersion and also consider scenarios that involve correlations between the stochastic weights.
\section*{Acknowledgments}
This work has been supported by the Australian Research Council (ARC) through grants DP160102401 and FT200100536, and by the South Australian Government through the Research Consortium "Unlocking Complex Resources through Lean Processing".

\bibliographystyle{apalike}

\bibliography{references.bib}

\end{document}